\documentclass{article} 
\usepackage{nips14submit_e,times}
\usepackage{hyperref}
\PassOptionsToPackage{hyphens}{url}
\usepackage{amsmath,amssymb,amsthm}
\usepackage{multirow}
\usepackage{graphicx}
\usepackage{subfigure}
\usepackage{wrapfig}
\usepackage{bbm}

\def \lmi {LMI}
\def \reals {\mathbb{R}}
\def \p {\mathbb{P}}

\def\decide#1#2{
  \mathrel{
    \mathop{
      \begin{array}{c}
        >\vspace{-1.5ex}\\<
      \end{array}
      }\limits_{#2}\limits^{#1}
    }
  }

\theoremstyle{plain} \newtheorem{thm}{Theorem}
 \newtheorem{lem}[thm]{Lemma}

\theoremstyle{definition} \newtheorem{defn}{Definition}

\theoremstyle{remark} 

\title{Efficient Minimax Signal Detection on Graphs}

\author{
Jing Qian \\
Division of Systems Engineering \\
Boston University \\
Brookline, MA 02446 \\
\texttt{jingq@bu.edu} \\
\And
Venkatesh Saligrama  \\
Department of Electrical and Computer Engineering\\
Boston University\\
Boston, MA 02215 \\
\texttt{srv@bu.edu} \\
}

\nipsfinalcopy 

\begin{document}

\maketitle

\begin{abstract}
Several problems such as network intrusion, community detection, and disease outbreak can be described by observations attributed to nodes or edges of a graph. In these applications presence of intrusion, community or disease outbreak is characterized by novel observations on some unknown connected subgraph. These problems can be formulated in terms of optimization of suitable objectives on connected subgraphs, a problem which is generally computationally difficult. We overcome the combinatorics of connectivity by embedding connected subgraphs into linear matrix inequalities (LMI). Computationally efficient tests are then realized by optimizing convex objective functions subject to these LMI constraints. We prove, by means of a novel Euclidean embedding argument, that our tests are minimax optimal for exponential family of distributions on 1-D and 2-D lattices. We show that internal conductance of the connected subgraph family plays a fundamental role in characterizing detectability. 
\end{abstract}

\section{Introduction}
Signals associated with nodes or edges of a graph arise in a number of applications including sensor network intrusion, disease outbreak detection and virus detection in communication networks. Many problems in these applications can be framed from the perspective of hypothesis testing between null and alternative hypothesis. Observations under null and alternative follow different distributions. The alternative is actually composite and identified by sub-collections of connected subgraphs.

To motivate the setup consider the disease outbreak problem described in \cite{patil03}. Nodes there are associated with counties and observations associated with each county correspond to reported cases of a disease. Under the null distribution, observations at each county are assumed to be poisson distributed and independent across different counties. Under the alternative there are a contiguous sub-collection of counties (connected sub-graph) that each experience elevated cases on average from their normal levels but are otherwise assumed to be independent. The eventual shape of the sub-collection of contiguous counties is highly unpredictable due to uncontrollable factors.



In this paper we develop a novel approach for signal detection on graphs that is both statistically effective and computationally efficient. Our approach is based on optimizing an objective function subject to subgraph connectivity constraints, which is related to generalized likelihood ratio tests (GLRT). GLRTs maximize likelihood functions over combinatorially many connected subgraphs, which is computationally intractable. On the other hand statistically, GLRTs have been shown to be asymptotically minimax optimal for exponential class of distributions on Lattice graphs \& Trees~\cite{Castro08} thus motivating our approach.
%
We deal with combinatorial connectivity constraints by obtaining a novel characterization of connected subgraphs in terms of convex Linear Matrix Inequalities (LMIs). 
In addition we show how our LMI constraints naturally incorporate other features such as shape and size. 
We show that the resulting tests are essentially minimax optimal for exponential family of distributions on 1-D and 2-D lattices. Conductance of the subgraph, a parameter in our LMI constraint, plays a central role in characterizing detectability. 




{\bf Related Work:}\, The literature on signal detection on graphs can be organized into parametric and non-parametric methods, which can be further sub-divided into computational and statistical analysis themes.  Parametric methods originated in the scan statistics literature \cite{Glaz02} with more recent work including that of \cite{Duczmal06a,Duczmal06b,Priebe06,patil03,lad12,vlad12} focusing on graphs. Much of this literature develops scanning methods that optimize over rectangles, circles or neighborhood balls \cite{Duczmal06b,Priebe06} across different regions of the graphs. However, the drawbacks of simple shapes and the need for non-parametric methods to improve detection power is well recognized. This has led to new approaches such as simulated annealing~\cite{Duczmal06b,Duczmal06a} but is lacking in statistical analysis. More recent work in ML literature~\cite{Qian14} describes semi-definite programming algorithm for non-parametric shape detection, which is similar to our work here. However, unlike us their method requires a heuristic rounding step, which does not lend itself to statistical analysis. In this context a number of recent papers have focused on statistical analysis \cite{Castro05,Castro08,Addaria10,Castro11} with non-parametric shapes. They derive fundamental bounds for signal detection for the elevated means testing problem in the Gaussian setting on special graphs such as trees and lattices. In this setting under the null hypothesis the observations are assumed to be independent identically distributed (IID) with standard normal random variables. Under the alternative the Gaussian random variables are assumed to be standard normal except on some {\it connected subgraph} where the mean $\mu$ is elevated. They show that GLRT achieves ``near''-minimax optimality in a number of interesting scenarios. While this work is interesting the suggested algorithms are computationally intractable. To the best of our knowledge only ~\cite{Sharpnack12b,Sharpnack13} explores a computationally tractable approach and also provides statistical guarantees. Nevertheless, this line of work does not explicitly deal with connected subgraphs (complex shapes) but deals with more general clusters. These are graph partitions with small out-degree. Although this appears to be a natural relaxation of connected subgraphs/complex-shapes it turns out to be quite loose\footnote{A connected subgraph on a 2-D lattice of size $K$ has out-degree at least $\Omega(\sqrt{K})$ while set of subgraphs with out-degree $\Omega(\sqrt{K})$ includes disjoint union of $\Omega(\sqrt{K}/4)$ nodes. So statistical requirements with out-degree constraints can be no better than those for arbitrary $K$-sets.}
and leads to substantial gap in statistical effectiveness for our problem.
In contrast we develop a new method for signal detection of complex shapes that is not only statistically effective but also computationally efficient.

\section{Problem Formulation}\label{sec:setup}
%
Let $G=(V,E)$ denote an undirected unweighted graph with $|V|=n$ nodes and $|E|=m$ edges. Associated with each node, $v \in V$, are observations $x_v \in \reals^p$. We assume observations are distributed $\p_0$ under the null hypothesis. The alternative is composite and the observed distribution, $\p_S$, is parameterized by $S \subseteq V$ belonging to a class of subsets $\Lambda \subseteq {\cal S}$, where ${\cal S}$ is the superset. We denote by ${\cal S}_K \subseteq {\cal S}$ the collection of size-$K$ subsets.
$E_S=\{ (u,v)\in E: u\in S, v\in S \}$ denotes the induced edge set on $S$.
We let $x_S$ denote the collection of random variables on the subset $S \subseteq V$. $S^c$ denotes nodes $V-S$. 
Our goal is to design a decision rule, $\pi$, that maps observations $x^n = (x_v)_{v\in V}$ to $\{0,\,1\}$ with zero denoting null hypothesis and one denoting the alternative. 
We formulate risk following the lines of \cite{Castro11} and combine Type I and Type II errors:
\begin{eqnarray}\label{eq:risks}
  R(\pi) &=& \mathbb{P}_0\left( \pi(x^n)=1 \right) + \max_{S\in \Lambda} \mathbb{P}_S\left( \pi(x^n)=0\right)
\end{eqnarray}
\begin{defn}[$\delta$-Separable]
We say that the composite hypothesis problem is $\delta$-separable if there exists a test $\pi$ such that, $R(\pi) \leq \delta$.
\end{defn}
We next describe asymptotic notions of detectability and separability. These notions requires us to consider large-graph limits. To this end we index a sequence of graphs $G_n=(V_n, E_n)$ with $n \rightarrow \infty$ and an associated sequence of tests $\pi_n$.
\begin{defn}[Separability]
We say that the composite hypothesis problem is asymptotically $\delta$-separable if there is some sequence of tests, $\pi_n$, such that $R(\pi_n) \leq \delta$ for sufficiently large $n$. It is said to be asymptotically separable if $R(\pi_n) \longrightarrow 0$. The composite hypothesis problem is said to be asymptotically inseparable if no such test exists.
\end{defn}
Sometimes, additional granular measures of performance are often useful to determine asymptotic behavior of Type I and Type II error. This motivates the following definition:
\begin{defn}[$\delta$-Detectability]
We say that the composite hypothesis testing problem is $\delta$-detectable if there is a sequence of tests, $\pi_n$, such that,
$$\sup_{S \in \Lambda } \p_S(\pi_n(x^n) = 0) \stackrel{n\rightarrow \infty} {\longrightarrow} 0,\,\,\,\limsup_n \p_0(\pi_n(x^n) = 1) \leq \delta$$
\end{defn}
In general $\delta$-detectability does not imply separability. For instance, consider $x \stackrel{H_0}{\sim} {\cal N}(0,\sigma^2)$ and $x \stackrel{H_1}{\sim} {\cal N}(\mu,{\frac{\sigma^2}{n}})$. It is $\delta$-detectable for ${\frac{\mu}{\sigma}} \geq 2\sqrt{\log{ \frac{1}{\delta}}}$ but not separable.
\paragraph{Generalized Likelihood Ratio Test (GLRT)}
is often used as a statistical test for composite hypothesis testing. Suppose $\phi_0(x^n)$ and $\phi_S(x^n)$ are probability density functions associated with $\p_0$ and $\p_S$ respectively. The GLRT test thresholds the ``best-case'' likelihood ratio, namely,
\begin{equation}
\mbox{GLRT:}\,\,\,\,\,\,\,\ell_{\max} (x^n) = \max_{S \in \Lambda } \ell_S(x^n) \decide{H_1}{H_0} \eta,\,\,\,\, \ell_S(x)=\log { \frac{\phi_S(x^n)}{\phi_0(x^n)}} \label{eq:glrt}
\end{equation}
\underline{\it Local Behavior}: Without additional structure, the likelihood ratio, $\ell_S(x)$ for a fixed $S \in \Lambda$ is
a function of observations across all nodes. Many applications exhibit {\it local behavior}, namely, the observations under the two hypothesis behave distinctly only on some small subset of nodes (as in disease outbreaks). This justifies introducing local statistical models in the following section.
\underline{\it Combinatorial}: The class $\Lambda$ is combinatorial such as collections of connected subgraphs and GLRT is not generally {\it computationally} tractable. On the other hand GLRT is minimax optimal for special classes of distributions and graphs and motivates development of tractable algorithms.
%
%

\subsection{Statistical Models \& Subgraph Classes}\label{sec:model}
The foregoing discussion motivates introducing local models, which we present next. Then informed by existing results on separability we categorize subgraph classes by shape, size and connectivity.

\subsubsection{Local Statistical Models}
\noindent
\underline{\it Signal in Noise Models} arise in sensor network (SNET) intrusion \cite{lad12,ermis10} and disease outbreak detection~\cite{patil03}. They are modeled with Gaussian (SNET) and Poisson (disease outbreak) distributions.
\begin{equation}\label{eq:H0H1_decay}
\mathbf{H_0:}\,\, \,\,x_v=w_v; \,\, \,\,\,\,\,\,  
\mathbf{H_1:}\,\, \,\,x_v=
\mu \alpha_{uv}\mathbf{1}_S(v) + w_v,\,\, \mbox{for some},\,\,\,S \in \Lambda,\,\,u\in S
\end{equation}
For Gaussian case we model $\mu$ as a constant, $w_v$ as IID standard normal variables, $\alpha_{uv}$ as the propagation loss from source node $u \in S$ to the node $v$. In disease outbreak detection $\mu =1$, $\alpha_{uv}\sim Pois(\lambda N_v)$ and $w_v \sim Pois(N_v)$ are independent Poisson random variables, and $N_v$ is the population of county $v$.
In these cases $\ell_S(x)$ takes the following local form where $Z_v$ is a normalizing constant.
\begin{equation} \label{indic_expr}
\ell_S(x) = \ell_S(x_S) \propto \sum_{v \in V} (\Psi_v(x_v) - \log(Z_v))\mathbf{1}_S(v)
\end{equation}
We characterize $\mu_0,\lambda_0$ as the minimum value that ensures separability for the different models:
\begin{equation} \label{musep}
\mu_0 = \inf \{ \mu \in \reals^+ \mid \exists \pi_n, \lim_{n \rightarrow \infty} R(\pi_n) =0\},\,\,\lambda_0 = \inf \{ \lambda \in \reals^+ \mid \exists \pi_n, \lim_{n \rightarrow \infty} R(\pi_n) =0\}
\end{equation}
\noindent
\underline{\it Correlated Models} arise in textured object detection~\cite{Cross1983} and protein subnetwork detection~\cite{Bailly2011}. For instance consider a common random signal $z$ on $S$, which results in uniform correlation $\rho >0$ on $S$.
\begin{equation}\label{eq:H0H1_correl}
\mathbf{H_0:}\,\, \,\,x_v=w_v; \,\, \,\,\,\,\,\,  
\mathbf{H_1:}\,\, \,\,x_v=
(\sqrt{\rho(1-\rho)^{-1}}) z\mathbf{1}_S(v) + w_v,\,\, \mbox{for some},\,\,\,S \in \Lambda,
\end{equation}
$z,\,w_v$ are standard IID normal random variables. Again we obtain $\ell_S(x)=\ell_S(x_S)$. These examples motivate the following general setup for local behavior:
\begin{defn}\label{def:local}
The distributions $\p_0$ and $\p_S$ are said to exhibit {\em local structure} if they satisfy:\\
\noindent
{\bf (1) Markovianity}: The null distribution $\p_0$ satisfies the properties of a Markov Random Field (MRF). Under the distribution $\p_S$ the observations $x_S$ are conditionally independent of $x_{S_1^c}$ when conditioned on annulus $S_1 \cap S^c$, where $S_1 = \{v \in V \mid d(v,w) \leq 1,\, w \in S\}$, is the 1-neighborhood of $S$. 
\noindent
{\bf (2) Mask}: Marginal distributions of observations under $\p_0$ and $\p_S$ on nodes in $S^c$ are identical:
$\p_0(x_{S^c} \in A)=\p_S(x_{S^c}\in A),\,\, \forall \, A \in {\cal A}$, the $\sigma$-algebra of measurable sets.
\end{defn}
\begin{lem}[\cite{lad12}]\label{lem:LR}
Under conditions (1) and (2) it follows that $\ell_S(x)=\ell_S(x_{S_1})$. 
\end{lem}
\subsubsection{Structured Subgraphs}
Existing works \cite{Castro05,Castro08,Castro11} point to the important role of size, shape and connectivity in determining detectability. For concreteness we consider the signal in noise model for Gaussian distribution and 
tabulate upper bounds from existing results for $\mu_0$ (Eq.~\ref{musep}). The lower bounds are messier and differ by logarithmic factors but this suffices for our discussion here.
The table reveals several important points. Larger sets are easier to detect -- $\mu_0$ decreases with size; connected $K$-sets are easier to detect relative to arbitrary $K$-sets; for 2-D lattices ``thick'' connected shapes are easier to detect than ``thin'' sets (paths); finally detectability on complete graphs is equivalent to arbitrary $K$-sets, i.e., shape does not matter. Intuitively, these tradeoffs make sense. For a constant $\mu$, ``signal-to-noise'' ratio increases with size. Combinatorially, there are fewer $K$-connected sets than arbitrary $K$-sets; fewer connected balls than connected paths; and fewer connected sets in 2-D lattices than dense graphs.
%
\begin{table}[htb]
\begin{tabular}{|c|c|c|c|}
\hline
 & Arbitrary $K$-Set & $K$-Connected Ball & $K$-Connected Path \\ \hline
Line Graph  & $\omega\left( \sqrt{2\log(n)} \right)$    & $\omega\left( \sqrt{\frac{2}{K}\log(n)} \right)$    & $\omega\left( \sqrt{\frac{2}{K}\log(n)} \right)$ \\ 
2-D Lattice   & $\omega\left( \sqrt{2\log(n)} \right)$    & $\omega\left( \sqrt{\frac{2}{K}\log(n)} \right)$    & $\omega\left( 1 \right)$ \\
Complete    & $\omega\left( \sqrt{2\log(n)} \right)$    & $\omega\left( \sqrt{2\log(n)} \right)$                & $\omega\left( \sqrt{2\log(n)} \right)$ \\ \hline
\end{tabular}
\end{table}
%
These results point to the need for characterizing the signal detection problem in terms of connectivity, size, shape and the properties of the ambient graph. We also observe that the table is somewhat incomplete. While balls can be viewed as thick shapes and paths as thin shapes, there are a plethora of intermediate shapes. A similar issue arises for sparse vs. dense graphs.  We introduce general definitions to categorize shape and graph structures below. 
\begin{defn}[Internal Conductance] (a.k.a. Cut Ratio)
Let $H=(S, F_S)$ denote a subgraph of $G=(V,E)$ where $S\subseteq V$, $F_S\subseteq E_S$, written as $H\subseteq G$. Define the internal conductance of $H$ as:
\begin{equation}\label{def:inner_conductance}
  \phi(H) = \min_{A \subset S}  \frac{ |\delta_S(A)| }{ \min\{|A|,|S-A|\} } ;\,\,\,\,
\delta_S(A) = \{(u,v) \in F_S \mid u \in A,\,v \in S-A\}
\end{equation}
\end{defn}
Apparently $\phi(H)=0$ if $H$ is not connected. The internal conductance of a collection of subgraphs, $\Sigma$, is defined as the smallest internal conductance:
\begin{equation*}
  \phi(\Sigma) = \min_{ H \in \Sigma }\phi(H)
\end{equation*}
For future reference we denote the collection of connected subgraphs by ${\cal C}$ and by ${\cal C}_{a,\Phi}$ the sub-collections containing node $a \in V$ with minimal internal conductance $\Phi$:
\begin{equation} \label{connected_set}
{\cal C}=\{ H \subseteq G: \phi(H)>0 \},\,\,\,{\cal C}_{a,\Phi}=\{H=(S,F_S) \subseteq G: a\in S, \phi(H)\geq \Phi\}
\end{equation}

In 2-D lattices, for example, $\phi(B_K) \approx \Omega({1/\sqrt{K}})$ for connected K-balls $B_K$ or other thick shapes of size $K$. $\phi( {\cal C}\cap {\cal S}_K ) \approx \Omega({1/K})$ due to ``snake''-like thin shapes. Thus internal conductance explicitly accounts for shape of the sets.

%
%


\section{Convex Programming}\label{sec:convex}
We develop a convex optimization framework for generating test statistics for local statistical models described in Section~\ref{sec:model}. Our approach relaxes the combinatorial constraints and the functional objectives of the GLRT problem of Eq.(\ref{eq:glrt}). In the following section we develop a new characterization based on linear matrix inequalities that accounts for size, shape and connectivity of subgraphs.
For future reference we denote $A \circ B \stackrel{\Delta}{=} [A_{ij} B_{ij}]_{i, j}$.

Our first step is to embed subgraphs, $H$ of $G$, into matrices. A binary symmetric incidence matrix, $A$, is associated with an undirected graph $G=(V,E)$, and encodes edge relationships. Formally, the edge set $E$ is the support of $A$, namely, $E = \mbox{Supp(A)}$. For subgraph correspondences we consider symmetric matrices, $M$, with components taking values in the unit interval, $[0,1]$.
$$
{\cal M} = \{M \in [0,1]^{n\times n} \mid M_{uv} \leq M_{uu},\,\, M\,\,\mbox{Symmetric}\}
$$
\begin{defn}\label{sub_embed}
$M\in {\cal M}$ is said to correspond to a subgraph $H=(S,F_S)$, written as $H \rightleftharpoons M$, if
$$
S = \mbox{Supp}\{\mbox{Diag}(M)\},\,\,F_S = \mbox{Supp}(A \circ M)
$$
\end{defn}
The role of $M\in {\cal M}$ is to ensure that if $u \not \in S$ we want the corresponding edges $M_{uv}=0$.
Note that $A \circ M$ in Defn.~\ref{sub_embed} removes the spurious edges $M_{uv} \not = 0$ for $(u,v) \notin E_S$.

Our second step is to characterize connected subgraphs as convex subsets of ${\cal M}$. Now a subgraph $H=(S,F_S)$ is a connected subgraph if for every $u, v \in S$, there is a path consisting only of edges in $F_S$ going from $u$ to $v$. This implies that for two subgraphs $H_1,\, H_2$ and corresponding matrices $M_1$ and $M_2$, their convex combination $M_{\eta} = \eta M_1 + (1-\eta)M_2,\,\,\eta \in (0,1)$ naturally corresponds to $H=H_1 \cup H_2$ in the sense of Defn~\ref{sub_embed}. On the other hand if $H_1 \cap H_2 = \emptyset$ then $H$ is disconnected and so $M_{\eta}$ is as well. This motivates our convex characterization with a common ``anchor'' node. To this end we consider the following collection of matrices:
$$
{\cal M}_a^*=\{M \in {\cal M} \mid M_{aa}=1,\,M_{vv} \leq M_{av}\}
$$
Note that ${\cal M}_a^*$ includes star graphs induced on subsets $S = \mbox{Supp}(\mbox{Diag}(M))$ with anchor node $a$. 
We now make use of the well known properties~\cite{Chung96} of the Laplacian of a graph to characterize connectivity. The unnormalized Laplacian matrix of an undirected graph $G$ with incidence matrix $A$ is described by $L(A) = \mbox{diag} (A \mathbf{1}_n) -A$ where $\mathbf{1}_n$ is the all-one vector.
\begin{lem}\label{lem_algebraic_connectivity}
Graph $G$ is connected if and only if the number of zero eigenvalues of $L(A)$ is one.
\end{lem}
Unfortunately, we cannot directly use this fact on the subgraph $A \circ M$ because there are many zero eigenvalues because the complement of $\mbox{Supp}(\mbox{Diag}(M))$ is by definition zero. We employ linear matrix inequalities (LMI) to deal with this issue. The condition~\cite{Boyd04}
$F(x) = F_0 + F_1 x_1 + \cdots + F_p x_p \succeq 0$
with symmetric matrices $F_j$ is called a linear matrix inequality in $x_j \in \reals$ with respect to the positive semi-definite cone represented by $\succeq$. Note that the Laplacian of the subgraph $L(A\circ M)$ is a linear matrix function of $M$. We denote a collection of subgraphs as follows: 
\begin{equation}\label{eq:C_LMI_def}
  {\cal C}_{\lmi}(a,\gamma) \stackrel{\Delta}{=} \{ H \rightleftharpoons M \mid M \in {\cal M}_a^*, \, L(A\circ M) - \gamma L(M) \succeq 0\}
\end{equation}

\begin{thm}\label{thm:connectivity}
The class ${\cal C}_{\lmi}(a,\gamma)$ is connected for $\gamma>0$. Furthermore, every connected subgraph can be characterized in this way for some $a \in V$ and $\gamma>0$, namely,
${\cal C}= \bigcup_{a \in V, \gamma>0} {\cal C}_{\lmi}(a,\gamma)$. 
\end{thm}

{\it Proof Sketch.}
$M \in {\cal C}_{\lmi}(a,\gamma)$ implies $M$ is connected. By definition of ${\cal M}_a$ there must be a star graph that is a subgraph on $\mbox{Supp}(\mbox{Diag}(M))$. This means that $L(M)$ (hence $L(A\circ M)$) can only have one zero eigenvalue on $\mbox{Supp}(\mbox{Diag}(M))$. We can now invoke Lemma~\ref{lem_algebraic_connectivity} on $\mbox{Supp}(\mbox{Diag}(M))$. The other direction is based on hyperplane separation of convex sets. Note that ${\cal C}_{a,\gamma}$ is convex but ${\cal C}$ is not.
This necessitates the need for an anchor. In practice this means that we have to search for connected sets with different anchors. This is similar to scan statistics the difference being that we can now optimize over arbitrary shapes. We next get a handle on $\gamma$.

\noindent
{\bf $\gamma$ encodes Shape:} \,
We will relate $\gamma$ to the internal conductance of the class ${\cal C}$. This provides us with a tool to choose $\gamma$ to reflect the type of connected sets that we expect for our alternative hypothesis. In particular thick sets correspond to relatively large $\gamma$ and thin sets to small $\gamma$. In general for graphs of fixed size the minimum internal conductance over all connected shapes is strictly positive and we can set $\gamma$ to be this value if we do not a priori know the shape. 
\begin{thm}\label{thm:C_a_Phi}
In a 2-D lattice, it follows that
${\cal C}_{a,\Phi} \subseteq {\cal C}_{\lmi}( a, \gamma )$,
where $\gamma = \Theta( \frac{\Phi^2}{\log (1/\Phi)} )$.
\end{thm}
%

\noindent
{\bf LMI-Test:}\, We are now ready to present our test statistics. 
We replace indicator variables with the corresponding matrix components in Eq.~\ref{indic_expr}, i.e., $\mathbf{1}_S(v) \rightarrow M_{vv},\,\,\mathbf{1}_S(u)\mathbf{1}_S(v) \rightarrow M_{uv}$ and obtain: 
\begin{eqnarray} \nonumber
\mbox{Elevated Mean:} \,&\ell_{M}(x) = \sum\limits_{v \in V} (\Psi_v(x_v) - \log(Z_v))M_{vv}\\ \label{eq:quadratic}
\mbox{Correlated Gaussian:} \,&\ell_M(x) \propto \sum\limits_{(u,v) \in E} \Psi(x_u,x_v)M_{uv} - \sum\limits_v M_{vv}\log (1-\rho)\\ \label{eq:LRTM}
\mbox{LMIT}_{a,\gamma}\,\,\,\,&\ell_{a,\gamma} (x) = \max\limits_{M \in {\cal C}_{\lmi}(a,\gamma)} \ell_M(x) \decide{H_1}{H_0} \eta 
\end{eqnarray}
This test explicitly makes use of the fact that alternative hypothesis is anchored at $a$ and the internal conductance parameter $\gamma$ is known. 
We will refine this test to deal with the completely agnostic case in the following section.

\section{Analysis}\label{sec:analysis}
In this section we analyze LMIT$_{a,\gamma}$ and the agnostic LMI tests for the Elevated Mean problem for exponential family of distributions on 2-D lattices.
For concreteness we focus on Gaussian \& Poisson models and derive lower and upper bounds for $\mu_0$ (see Eq.~\ref{musep}). Our main result states that to guarantee separability, $\mu_0 \approx \Omega \left( \frac{1}{K \Phi} \right)$, where $\Phi$ is the internal conductance of the family ${\cal C}_{a,\Phi}$ of connected subgraphs, $K$ is the size of the subgraphs in the family, and $a$ is some node that is common to all the subgraphs.
%
The reason for our focus on homogenous Gaussian/Poisson setting is that we can extend current lower bounds in the literature to our more general setting and demonstrate that they match the bounds obtained from our LMIT analysis. We comment on how our LMIT analysis extends to other general structures and models later. 

The proof for LMIT analysis involves two steps (see Supplementary):
\begin{enumerate}
  \item {\it Lower Bound:} Under $H_1$ we show that the ground truth is a feasible solution. This allows us to lower bound the objective value, $\ell_{a,\gamma}(x)$, of Eq.~\ref{eq:LRTM}.
  \item {\it Upper Bound:} Under $H_0$ we consider the dual problem. By weak duality it follows that any feasible solution of the dual is an upper bound for $\ell_{a,\gamma}(x)$. A dual feasible solution is then constructed through a novel Euclidean embedding argument.
\end{enumerate}
We then compare the upper and lower bounds to obtain the critical value $\mu_0$.


We analyze both non-agnostic and agnostic LMI tests for the homogenous version of Gaussian and Poisson models of Eq.~\ref{eq:H0H1_decay} for both finite and asymptotic 2-D lattice graphs. For the finite case the family of subgraphs in Eq.~\ref{eq:H0H1_decay} is assumed to belong to the connected family of sets, ${\cal C}_{a,\Phi}\cap {\cal S}_K$, containing a fixed common node $a\in V$ of size $K$. For the asymptotic case we let the size of the graph approach infinity ($n \rightarrow \infty$). For this case we consider a sequence of connected family of sets ${\cal C}^n_{a.\Phi_n}\cap {\cal S}_{K_n}$ on graph $G_n=(V_n,E_n)$ with some fixed anchor node $a \in V_n$. We will then describe results for agnostic LMI tests, i.e., lacking knowledge of conductance $\Phi$ and anchor node $a$.

\noindent
{\bf Poisson Model:}\,
In Eq.~\ref{eq:H0H1_decay} we let the population $N_v$ to be identically equal to one across counties. We present LMI tests that are {\it agnostic} to shape and anchor nodes:
\begin{equation} \label{eq:LMIT_pois}
\mbox{LMIT}_A: \,\,\,\,\ell(x)=\max_{a \in V, \gamma \geq \Phi^2_{min}} \sqrt{\gamma} \ell_{a,\gamma} (x)  \decide{H_0}{H_1} 0
\end{equation}
where $\Phi_{min}$ denotes the minimum possible conductance of a connected subgraph with size $K$, which is $2/K$.
\begin{thm}
The $\mbox{LMIT}_{a,\gamma}$ test achieves $\delta$-separability for $\lambda = \Omega(\frac{ \log(K) }{K \Phi})$ and the agnostic test LMIT$_A$ for $\lambda =\Omega(\log K \sqrt{\log n})$.
\end{thm}
Next we consider the asymptotic case and characterize tight bounds for separability. 
\begin{thm}
The two hypothesis $H_0$ and $H_1$ are asymptotically inseparable if $\lambda_n \Phi_n K_n \log(K_n) \rightarrow 0$. It is asymptotically separable with $\mbox{LMIT}_{a,\gamma}$ for $\lambda_n K_n \Phi_n / \log(K_n) \rightarrow \infty$.
The agnostic $\mbox{LMIT}_A$ achieves asymptotic separability with $\lambda_n / ( \log(K_n) \sqrt{ \log n } )  \rightarrow\infty$.
\end{thm}

\noindent
{\bf Gaussian Model:}\,
We next consider agnostic tests for Gaussian model of Eq.~\ref{eq:H0H1_decay} with no propagation loss, i.e., $\alpha_{uv}=1$.

\begin{thm}\label{thm:gauss}
The two hypotheses $H_0$ and $H_1$ for the Gaussian model are asymptotically inseparable if $\mu_n \Phi_n K_n \log(K_n) \rightarrow 0$, are separable with $\mbox{LMIT}_{a,\gamma}$ if $\mu_n K_n \Phi_n /\log(K_n) \rightarrow \infty$, and are separable with $\mbox{LMIT}_A$ if $\mu_n / ( \log(K_n) \sqrt{ \log n  } ) \rightarrow \infty$
\end{thm}
Our inseparability bound matches existing results on 2-D Lattice \& Line Graphs by plugging in appropriate values for $\Phi$ for the cases considered in \cite{Castro08,Castro11}. The lower bound is obtained by specializing to a collection of ``non-decreasing band'' subgraphs.
Yet LMIT$_{a,\gamma}$ and LMIT$_A$ is able to achieves the lower bound within a logarithmic factor.
Furthermore, our analysis extends beyond Poisson \& Gaussian models and applies to general graph structures and models.
The main reason is that our LMIT analysis is fairly general and provides an observation-dependent bound through convex duality. We briefly describe it here. Consider functions $\ell_S(x)$ that are positive, separable and bounded for simplicity. By establishing primal feasibility that the subgraph $S \in {\cal C}_{LMI}(a,\gamma)$ for a suitably chosen $\gamma$, we can obtain a lower bound for the alternative hypothesis $H_1$ and show that $E_{H_1} \left ( \max_{M \in {\cal C}_{LMI}(a,\gamma)} \ell_M(x) \right )\geq E_{H_1}\left ( \sum_{v \in S} \ell_S(x_v)\right )$. On the other hand for the null hypothesis we can show that, $E_{H_0}\left (\max_{M \in {\cal C}_{LMI}(a,\gamma)} \ell_M(x) \right )\leq E_{H_0}\left ( \sum_{v \in B(a,\Theta(\sqrt{\gamma}))} \ell_S(x_v)\right )$.
Here $E_{H_1}$ and $E_{H_0}$ denote expectations with respect to alternative and null hypothesis and $B(a,\Theta(\sqrt{\gamma}))$ is a ball-like thick shape centered at $a \in V$ with radius $\Theta(\sqrt{\gamma})$. Our result then follows by invoking standard concentration inequalities.
We can extend our analysis to the non-separable case such as correlated models because of the linear objective form in Eq.~\ref{eq:quadratic}.

\section{Experiments} \label{sec:experiment}
We present several experiments to highlight key properties of LMIT and to compare LMIT against other state-of-art parametric and non-parametric tests on synthetic and real-world data. We have shown that agnostic LMIT is near minimax optimal in terms of asymptotic separability. However, separability is an asymptotic notion and only characterizes the special case of zero false alarms (FA) and missed detections (MD), which is often impractical. It is unclear how LMIT behaves with finite size graphs when FAs and MDs are prevalent. In this context incorporating priors could indeed be important. Our goal is to highlight how shape prior (in terms of thick, thin, or arbitrary shapes) can be incorporated in LMIT using the parameter $\gamma$ to obtain better AUC performance in finite size graphs. Another goal is to demonstrate how LMIT behaves with denser graph structures.

From the practical perspective, our main step is to solve the following SDP problem:
\begin{equation*}
  \max_{M}: \,\, \sum_i y_i M_{ii}  \,\,\,\,\,\,\,\,\,\,\,\,  s.t. \,\,\,\, M\in \mathcal{C}_{LMI}(a,\gamma), \,\,\,\, tr(M)\leq K
\end{equation*}
We use standard SDP solvers which can scale up to $n\sim 1500$ nodes for sparse graphs like lattice and $n\sim 300$ nodes for dense graphs with $m=\Theta(n^2)$ edges.

%
%
%
%
\begin{figure*}[!tb]\label{fig:shapes_demo}
\begin{centering}
\begin{minipage}[t]{.24\textwidth}
\includegraphics[width = 1\textwidth]{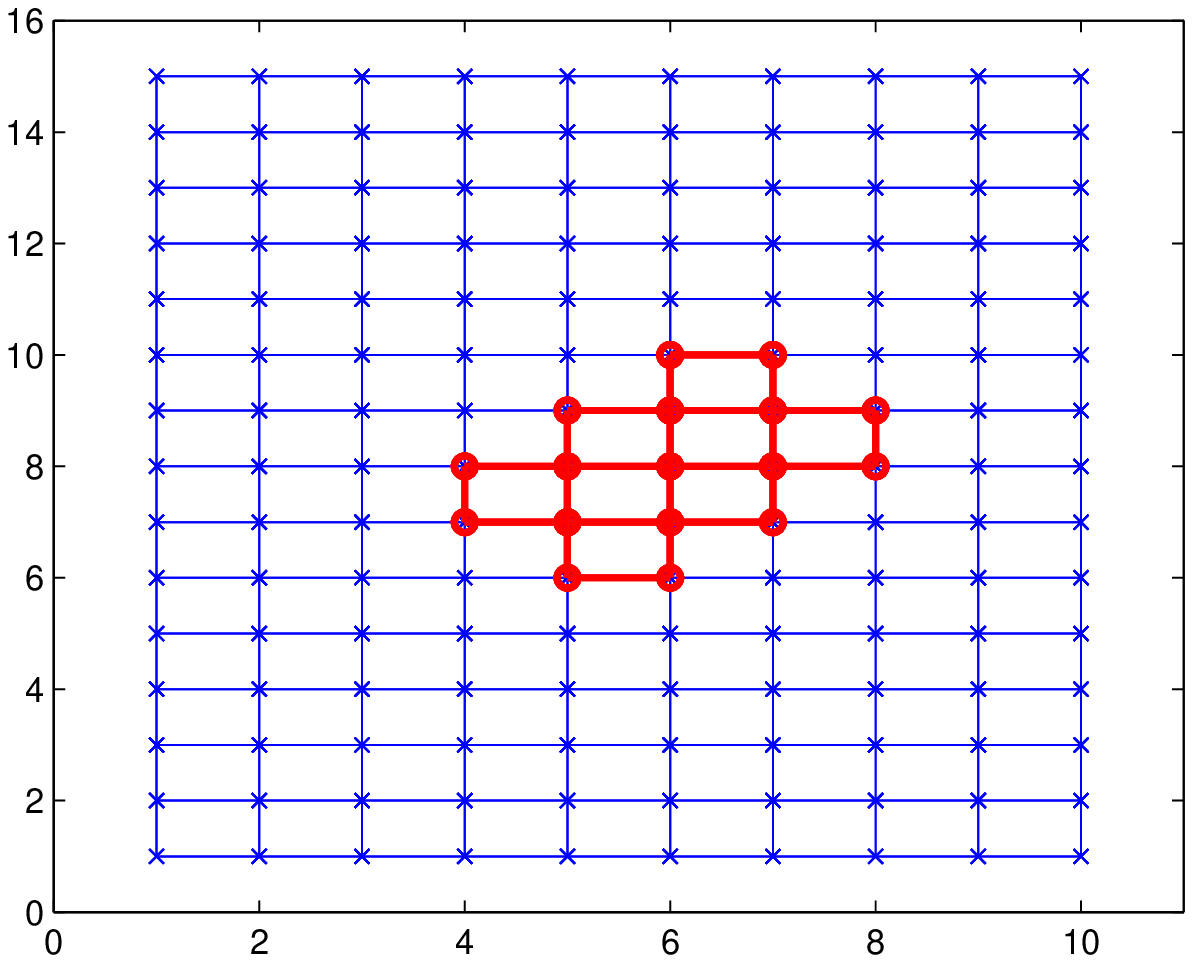}
\makebox[3.5 cm]{\small (a) Thick shape}
\end{minipage}
\begin{minipage}[t]{.24\textwidth}
\includegraphics[width = 1\textwidth]{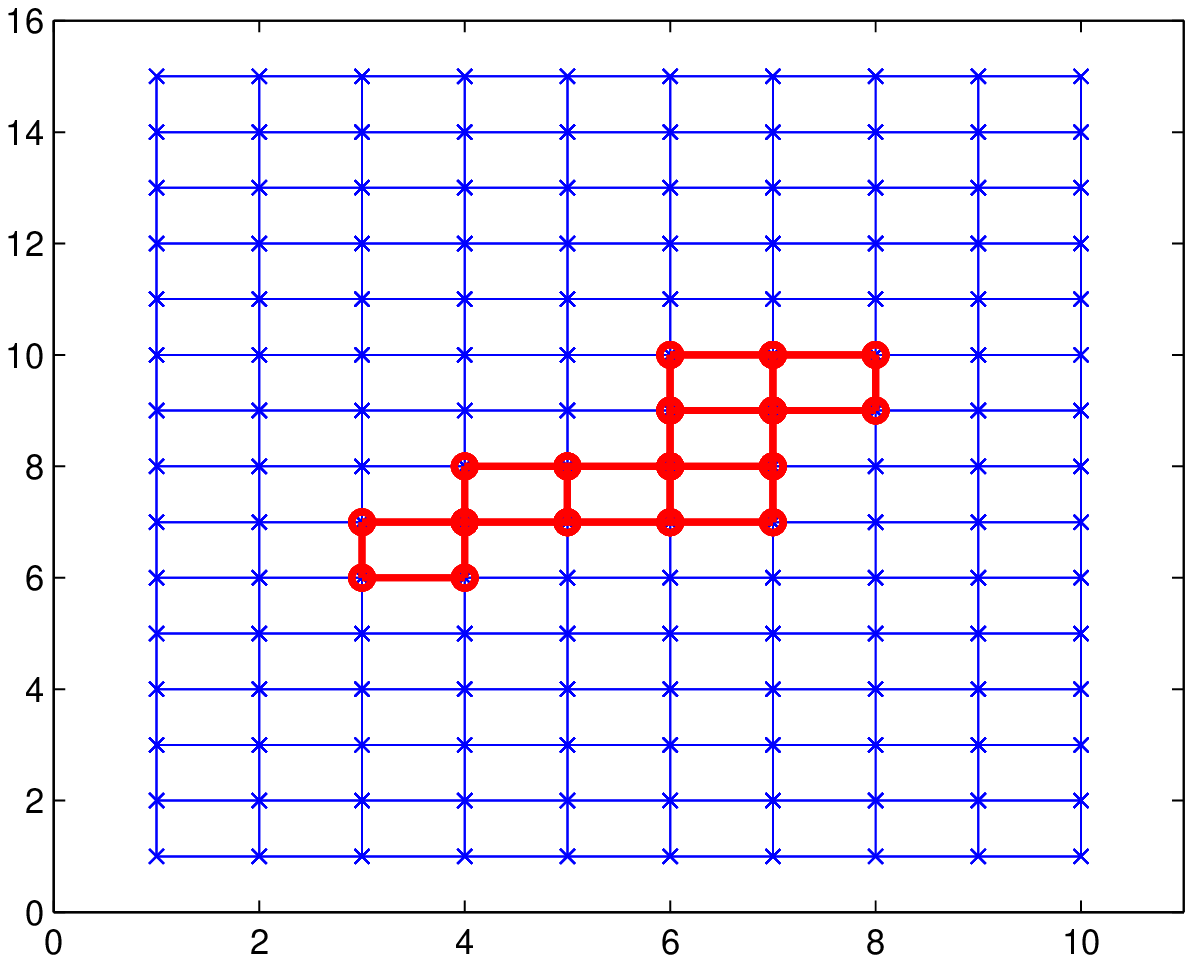}
\makebox[3.5 cm]{\small (b) Thin shape }
\end{minipage}
\begin{minipage}[t]{.24\textwidth}
\includegraphics[width = 1\textwidth]{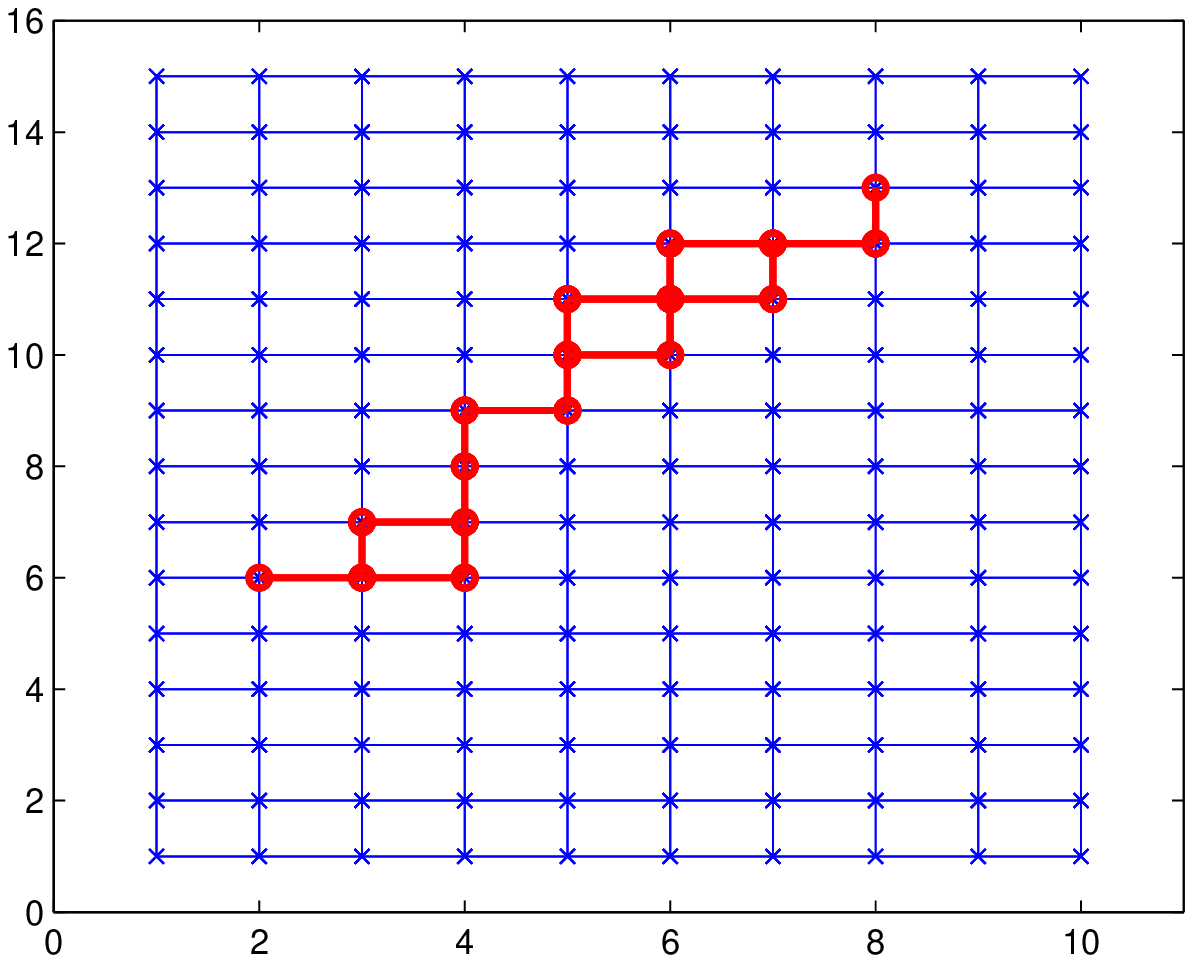}
\makebox[3.5 cm]{\small (c) Snake shape }
\end{minipage}
\begin{minipage}[t]{.24\textwidth}
\includegraphics[width = 1\textwidth]{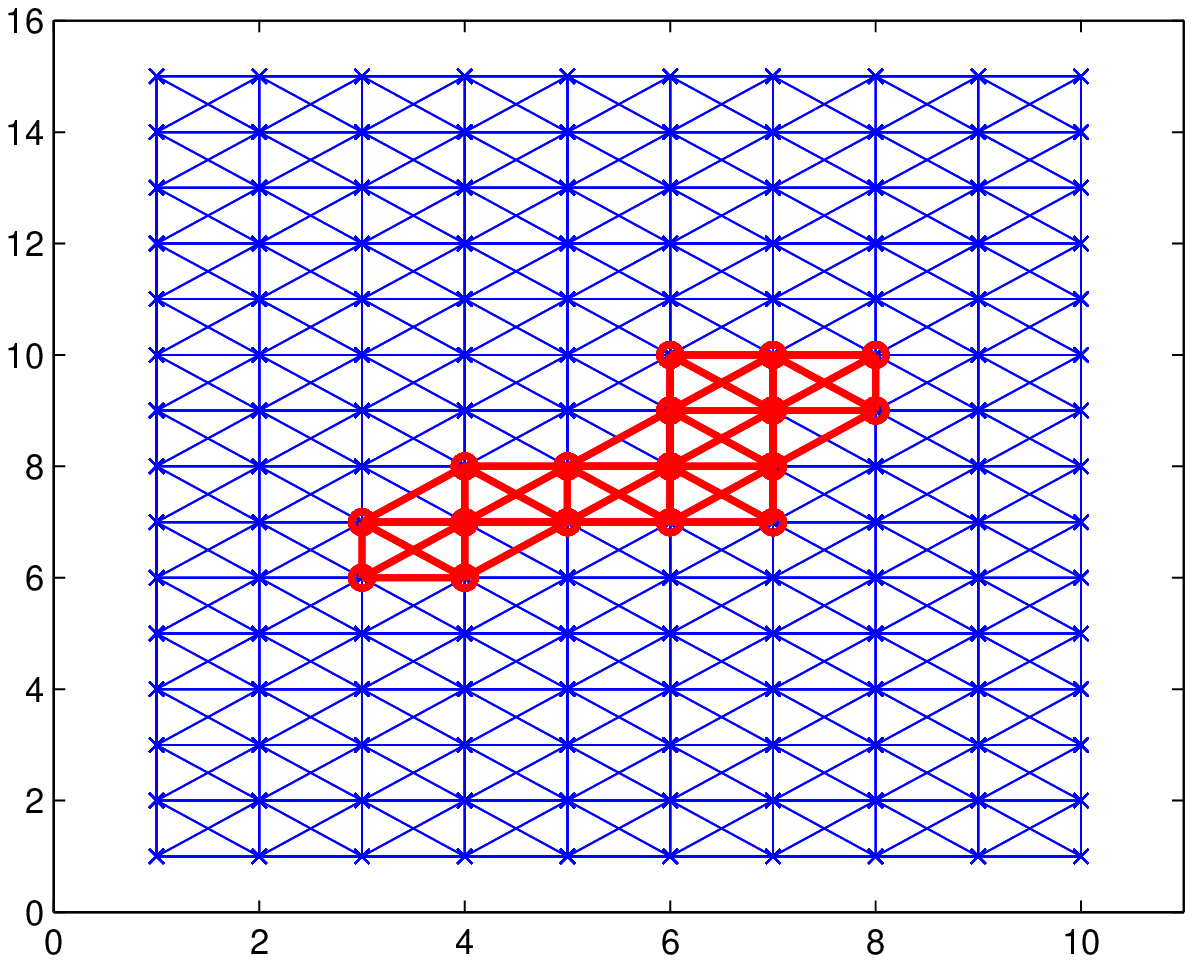}
\makebox[3.5 cm]{\small (d) Thin shape(8-neighbors) }
\end{minipage}
\caption{\small Various shapes of ground-truth anomalous clusters on a fixed 15$\times$10 lattice. Anomalous cluster size is fixed at 17 nodes. (a) shows a thick cluster with a large internal conductance. (b) shows a relatively thinner shape. (c) shows a snake-like shape which has the smallest internal conductance. (d) shows the same shape of (b), with the background lattice more densely connected. }
\end{centering}
\vspace*{-0.2in}
\end{figure*}

To understand the impact of shape we consider the test LMIT$_{a,\gamma}$ for Gaussian model and manually vary $\gamma$.
On a 15$\times$10 lattice we fix the size (17 nodes) and the signal strength $\mu\sqrt{|S|}=3$, and consider three different shapes (see Fig.~1) for the alternative hypothesis. For each shape we synthetically simulate 100 null and 100 alternative hypothesis and plot AUC performance of LMIT as a function of $\gamma$. We observe that the optimum value of AUC for thick shapes is achieved for large $\gamma$ and small $\gamma$ for thin shape confirming our intuition that $\gamma$ is a good surrogate for shape. In addition we notice that thick shapes have superior AUC performance relative to thin shapes, again confirming intuition of our analysis.
\begin{figure*}[!tb]\label{fig:AUC_demo}
\begin{centering}
\begin{minipage}[t]{.48\textwidth}
\includegraphics[width = 1\textwidth]{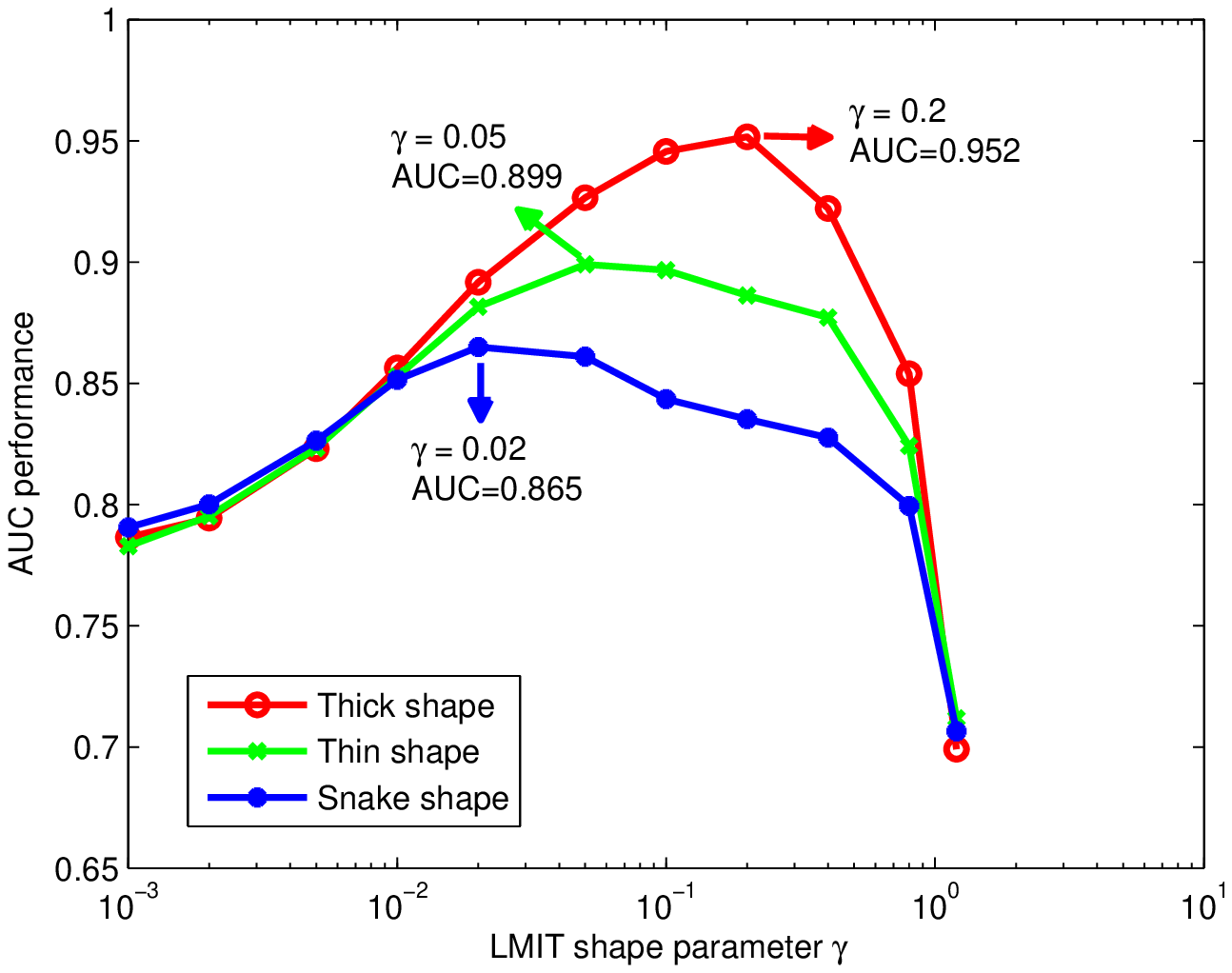}
\makebox[6.5 cm]{\small (a) AUC with various shapes}
\end{minipage}
\begin{minipage}[t]{.48\textwidth}
\includegraphics[width = 1\textwidth]{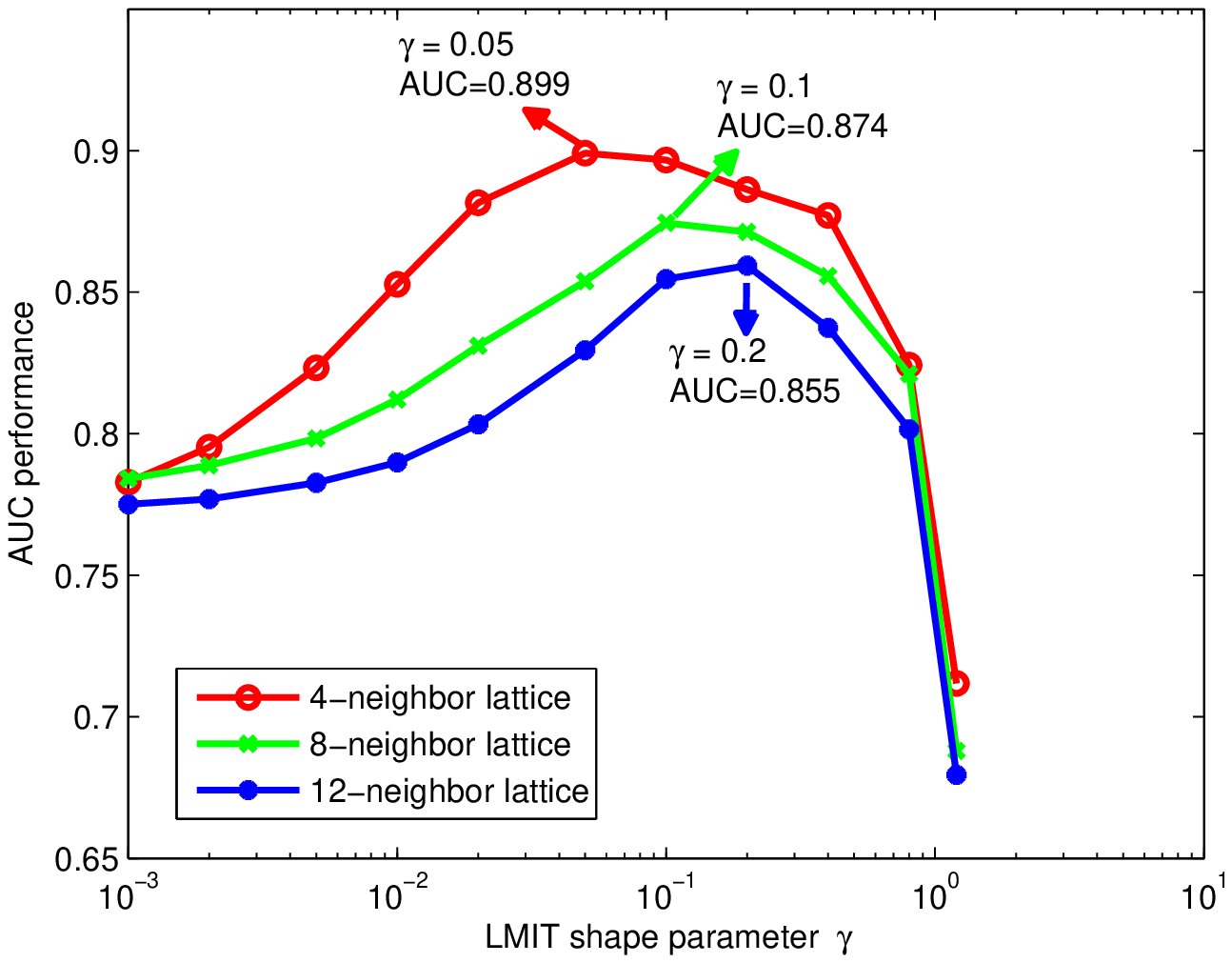}
\makebox[6.5 cm]{\small (b) AUC with different graph structures }
\end{minipage}
\caption{\small (a) demonstrates AUC performances with fixed lattice structure, signal strength $\mu$ and size (17 nodes), but different shapes of ground-truth clusters, as shown in Fig.1. (b) demonstrates AUC performances with fixed signal strength $\mu$, size (17 nodes) and shape (Fig.1(b)), but different lattice structures.}
\end{centering}
\vspace*{-0.25in}
\end{figure*}

To understand the impact of dense graph structures we consider performance of LMIT with neighborhood size. On the lattice of the previous experiment we vary neighborhood by connecting each node to its 1-hop, 2-hop, and 3-hop neighbors to realize denser structures with each node having 4, 8 and 12 neighbors respectively. Note that all the different graphs have the same vertex set. This is convenient because we can hold the shape under the alternative fixed for the different graphs. As before we generate 100 alternative hypothesis using the thin set of the previous experiment with the same mean $\mu$ and 100 nulls. The AUC curves for the different graphs highlight the fact that higher density leads to degradation in performance as our intuition with complete graphs suggests. We also see that as density increases a larger $\gamma$ achieves better performance confirming our intuition that as density increases the internal conductance of the shape increases.


\begin{table}[b]
\vspace*{-0.2in}
\caption{\small AUC performance of various algorithms on a 300-node lattice, a 200-node RGG, and the county map graph. On all three graphs LMIT significantly outperforms the other tests consistently for all SNR levels.}
\begin{center}
\begin{tabular}{|c||c|c|c||c|c|c||c|c|c|}
  \hline
  \multicolumn{1}{|c||}{\multirow{2}{*}{SNR}}  &   \multicolumn{3}{c||}{lattice ($\mu\sqrt{|S|}/\sigma$)} &   \multicolumn{3}{c||}{RGG ($\mu\sqrt{|S|}/\sigma$)} &   \multicolumn{3}{c|}{map ($\lambda_1/\lambda_0$)}\\
  \cline{2-10}
  \multicolumn{1}{|c||}{}  & 1.5 & 2 & 3            & 1.5 & 2 & 3           & 1.1 & 1.3 & 1.5 \\
  \hline\hline
     LMIT       & {\bf 0.728} & {\bf 0.780} & {\bf 0.882}    & {\bf 0.642} & {\bf 0.723} & {\bf 0.816}    & {\bf 0.606} & {\bf 0.842} & {\bf 0.948} \\
     SA         & 0.672 & 0.741 & 0.827                      & 0.627 & 0.677 & 0.756                      & 0.556 & 0.744 & 0.854 \\
     Rect(NB)   & 0.581 & 0.637 & 0.748                      & 0.584 & 0.632 & 0.701                      & 0.514 & 0.686 & 0.791 \\
     MaxT       & 0.531 & 0.547 & 0.587                      & 0.529 & 0.562 & 0.624                      & 0.525 & 0.559 & 0.543 \\
     AvgT       & 0.565 & 0.614 & 0.705                      & 0.545 & 0.623 & 0.690                      & 0.536 & 0.706 & 0.747 \\
  \hline
\end{tabular}
\end{center}
\label{tab:AUC}
\end{table}

In this part we compare LMIT against existing state-of-art approaches on a 300-node lattice, a 200-node random geometric graph (RGG), and a real-world county map graph (129 nodes) (see Fig.\ref{fig:lattice_RGG_demo},\ref{fig:irregular_general}).
We incorporate shape priors by setting $\gamma$ (internal conductance) to correspond to thin sets. While this implies some prior knowledge, we note that this is not necessarily the optimal value for $\gamma$ and we are still agnostic to the actual ground truth shape (see Fig.\ref{fig:lattice_RGG_demo},\ref{fig:irregular_general}).
For the lattice and RGG we use the elevated-mean Gaussian model. Following \cite{patil03} we adopt an elevated-rate independent Poisson model for the county map graph. Here $N_i$ is the population of county, $i$. Under null the number of cases at county $i$, follows a Poisson distribution with rate $N_i \lambda_0$ and under the alternative a rate $N_i\lambda_1$ within some connected subgraph. We assume $\lambda_1>\lambda_0$ and apply a weighted version of LMIT of Eq.~\ref{eq:LMIT_pois}, which arises on account of differences in population. We compare LMIT against several other tests, including simulated annealing (SA) \cite{Duczmal06a}, rectangle test (Rect), nearest-ball test (NB), and two naive tests: maximum test (MaxT) and average test (AvgT). SA is a non-parametric test and works by heuristically adding/removing nodes toward a better normalized GLRT objective while maintaining connectivity. Rect and NB are parametric methods with Rect scanning rectangles on lattice and NB scanning nearest-neighbor balls around different nodes for more general graphs (RGG and county-map graph). MaxT \& AvgT are often used for comparison purposes. MaxT is based on thresholding the maximum observed value while AvgT is based on thresholding the average value.

We observe that uniformly MaxT and AvgT perform poorly. This makes sense; It is well known that MaxT works well only for alternative of small size while AvgT works well with relatively large sized alternatives ~\cite{Addaria10}. Parametric methods (Rect/NB) performs poorly because the shape of the ground truth under the alternative cannot be well-approximated by Rectangular or Nearest Neighbor Balls. Performance of SA requires more explanation. One issue could be that SA does not explicitly incorporate shape and directly searches for the best GLRT solution. We have noticed that this has the tendency to amplify the objective value of null hypothesis because SA exhibits poor ``regularization'' over the shape. On the other hand LMIT provides some regularization for thin shape and does not admit arbitrary connected sets.

\newpage

\bibliographystyle{unsrt}
\bibliography{ACD_bib}

\begin{thebibliography}{10}

\bibitem{patil03}
G.~P. Patil and C.~Taillie.
\newblock Geographic and network surveillance via scan statistics for critical
  area detection.
\newblock In {\em Statistical Science}, volume 18(4), pages 457--465, 2003.

\bibitem{Castro08}
E.~Arias-Castro, E.~J. Candes, H.~Helgason, and O.~Zeitouni.
\newblock Searching for a trail of evidence in a maze.
\newblock In {\em The Annals of Statistics}, volume 36(4), pages 1726--1757,
  2008.

\bibitem{Glaz02}
J.~Glaz, J.~Naus, and S.~Wallenstein.
\newblock {\em Scan Statistics}.
\newblock Springer, New York, 2001.

\bibitem{Duczmal06a}
L.~Duczmal and R.~Assuncao.
\newblock A simulated annealing strategy for the detection of arbitrarily
  shaped spatial clusters.
\newblock In {\em Computational Statistics and Data Analysis}, volume~45, pages
  269--286, 2004.

\bibitem{Duczmal06b}
M.~Kulldorff, L.~Huang, L.~Pickle, and L.~Duczmal.
\newblock An elliptic spatial scan statistic.
\newblock In {\em Statistics in Medicine}, volume~25, 2006.

\bibitem{Priebe06}
C.~E. Priebe, J.~M. Conroy, D.~J. Marchette, and Y.~Park.
\newblock Scan statistics on enron graphs.
\newblock In {\em Computational and Mathematical Organization Theory}, 2006.

\bibitem{lad12}
V.~Saligrama and M.~Zhao.
\newblock Local anomaly detection.
\newblock In {\em Artificial Intelligence and Statistics}, volume~22, 2012.

\bibitem{vlad12}
V.~Saligrama and Z.~Chen.
\newblock Video anomaly detection based on local statistical aggregates.
\newblock {\em 2013 IEEE Conference on Computer Vision and Pattern
  Recognition}, 0:2112--2119, 2012.

\bibitem{Qian14}
J.~Qian and V.~Saligrama.
\newblock Connected sub-graph detection.
\newblock In {\em International Conference on Artificial Intelligence and
  Statistics (AISTATS)}, 2014.

\bibitem{Castro05}
E.~Arias-Castro, D.~Donoho, and X.~Huo.
\newblock Near-optimal detection of geometric objects by fast multiscale
  methods.
\newblock In {\em IEEE Transactions on Information Theory}, volume 51(7), pages
  2402--2425, 2005.

\bibitem{Addaria10}
Addario-Berry, N.~Broutin, L.~Devroye, and G.~Lugosi.
\newblock On combinatorial testing problems.
\newblock In {\em The Annals of Statistics}, volume 38(5), pages 3063--3092,
  2010.

\bibitem{Castro11}
E.~Arias-Castro, E.~J. Candes, and A.~Durand.
\newblock Detection of an anomalous cluster in a network.
\newblock In {\em The Annals of Statistics}, volume 39(1), pages 278--304,
  2011.

\bibitem{Sharpnack12b}
J.~Sharpnack, A.~Rinaldo, and A.~Singh.
\newblock Changepoint detection over graphs with the spectral scan statistic.
\newblock In {\em International Conference on Artificial Intelligence and
  Statistics}, 2013.

\bibitem{Sharpnack13}
J.~Sharpnack, A.~Krishnamurthy, and A.~Singh.
\newblock Near-optimal anomaly detection in graphs using lovasz extended scan
  statistic.
\newblock In {\em Neural Information Processing Systems}, 2013.

\bibitem{ermis10}
Erhan~Baki Ermis and Venkatesh Saligrama.
\newblock Distributed detection in sensor networks with limited range
  multimodal sensors.
\newblock {\em {IEEE} Transactions on Signal Processing}, 58(2):843--858, 2010.

\bibitem{Cross1983}
G.~R. Cross and A.~K. Jain.
\newblock Markov random field texture models.
\newblock In {\em IEEE Transactions on Pattern Analysis and Machine
  Intelligence}, volume~5, pages 25--39, 1983.

\bibitem{Bailly2011}
M.~Bailly-Bechet, C.~Borgs, A.~Braunstein, J.~T. Chayes, A.Dagkessamanskaia,
  J.~Francois, and R.~Zecchina.
\newblock Finding undetected protein associations in cell signaling by belief
  propagation.
\newblock In {\em Proceedings of the National Academy of Sciences (PNAS)},
  volume 108, pages 882--887, 2011.

\bibitem{Chung96}
F.~Chung.
\newblock {\em Spectral graph theory}.
\newblock American Mathematical Society, 1996.

\bibitem{Boyd04}
S.~Boyd and L.~Vandenberghe.
\newblock {\em Convex Optimization}.
\newblock Cambridge University Press, 2004.

\end{thebibliography}


\newpage

\section*{Appendix: Proofs of Theorems}

\begin{figure}[htb]
\begin{centering}
\begin{minipage}[t]{.44\textwidth}
\includegraphics[width=1\textwidth]{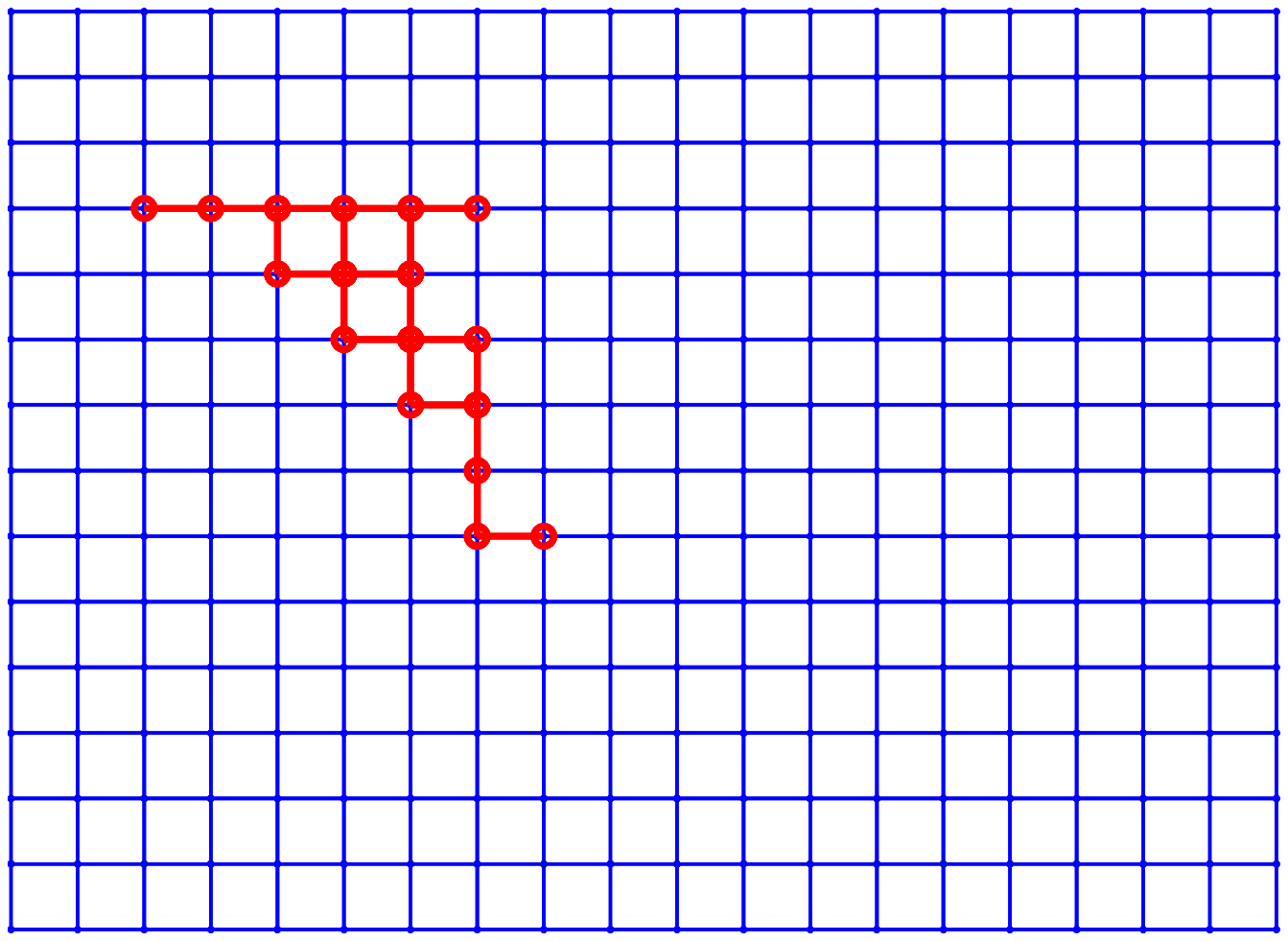}
\end{minipage}
\begin{minipage}[t]{.44\textwidth}
\includegraphics[width=1\textwidth]{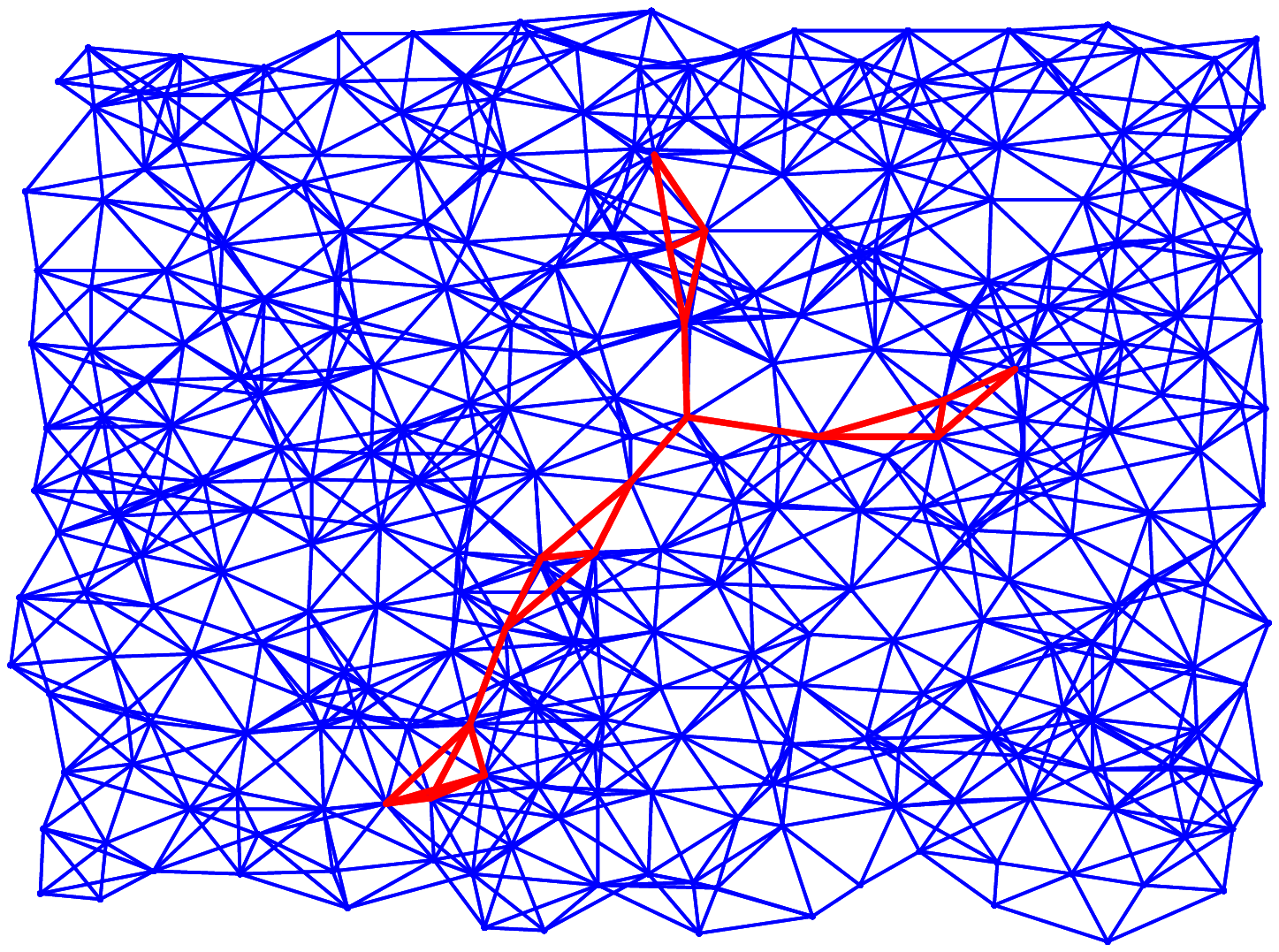}
\end{minipage}p
\caption{ 300-node lattice and 200-node RGG with 17-node anomalous cluster. \label{fig:lattice_RGG_demo} }
\end{centering}
\end{figure}

\begin{figure}[htb]
\begin{centering}
\begin{minipage}[t]{.6\textwidth}
\includegraphics[width=1\textwidth]{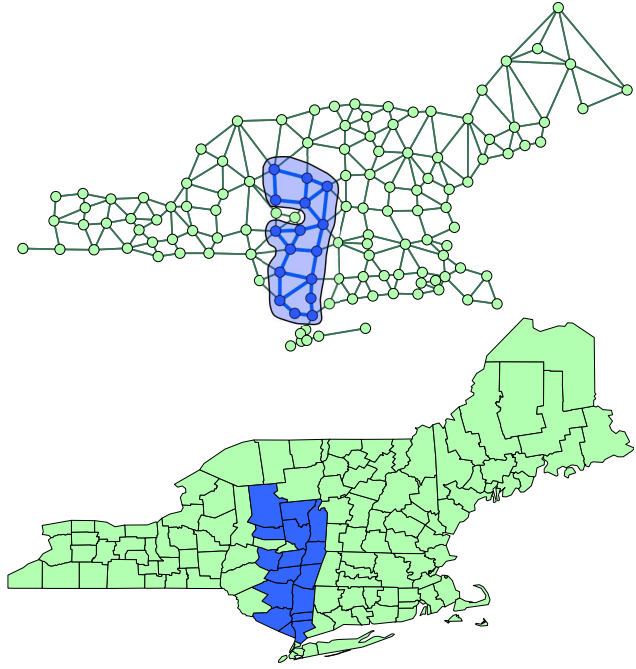}
\end{minipage}
\caption{ County map, graph representation and ground truth anomalous cluster for Table 1. }
\label{fig:irregular_general}
\end{centering}
\end{figure}

%


{\bf Proof of Theorem \ref{thm:connectivity}:}
\begin{proof}
For the first part we show $\forall a\in V, \gamma>0$, ${\cal C}_{LMI}(a,\gamma) \subseteq {\cal C}$.
Let $H=(S,F_S)\in {\cal C}_{LMI}(a,\gamma)$ be a connected subgraph.
Assume on the contrary that $H$ is disconnected: $S=C\cup \bar{C}$, where $\bar{C}=S-C$. Let $|S|=k,|C|=k_1,|\bar{C}|=k_2$. W.l.o.g. assume $a=1$, i.e. $M_{11}=1$, and $C$ consists of nodes $\{1,2,...,k_1\}$.

Let $Q(M;\gamma)=L(A\circ M) - \gamma L(M)$.
Consider the $k\times k$ sub-matrix $Q_S$ of $Q$ corresponding to $S$, since the rest part are all 0.
Now we use the vector $g = [\textbf{1}_{k_1}; -\textbf{1}_{k_2}]$ to hit $Q_S$:
\begin{equation}
  g' Q_S g = g' L_S(A_S\circ M_S) g - \gamma g' L_S(M_S) g \geq 0.
\end{equation}
Note that $A_S$ has the form:
\begin{equation}
  A_S = \left(
              \begin{array}{cc}
                A_C & 0 \\
                0 & A_{\bar{C}} \\
              \end{array}
            \right),
\end{equation}
where the off-diagonal block is zero because by assumption $C$ and $\bar{C}$ is disconnected. Then:
\begin{equation}
  L_S(A_S\circ M_S) = Diag\left( (A_S\circ M_S)\textbf{1}_n\right) - (A_S\circ M_S) = \left(
              \begin{array}{cc}
                \tilde{L}_C & 0 \\
                0 & \tilde{L}_{\bar{C}} \\
              \end{array}
            \right),
\end{equation}
where $\tilde{L}_C$ is the Laplacian matrix of $C$ weighted by $M_C$. Notice it still holds that $\tilde{L}_C \textbf{1}_{k_1}=0$. This means $g' L_S(A_S\circ M_S) g = 0$.

On the other hand, let $L_S(M_S)$ be:
\begin{equation}
  L_S(M_S) = Diag\left( M_S\textbf{1}_n\right) - M_S = \left(
              \begin{array}{cc}
                L_1 & L_3 \\
                L_3' & L_2 \\
              \end{array}
            \right).
\end{equation}
Using $g_1 = [\textbf{1}_{k_1}; 0]$ and $g_2 = [0; \textbf{1}_{k_2}]$ to hit $Q_S$ will yield: $\textbf{1}_{k_1}' L_1 \textbf{1}_{k_1} = 0$ and $\textbf{1}_{k_2}' L_2 \textbf{1}_{k_2} = 0$.
Apparently $g' L_S(M_S) g \geq 0$ due to positive semi-definiteness of Laplacian matrix. If it's strictly positive, proof is done. Otherwise this means $\textbf{1}_{k_1}' L_3 \textbf{1}_{k_2} = 0$.
Note that all entries of $L_3$ are either 0 or negative due to non-negativity of $M_S$. This means $L_3=0$, or equivalently $M_{ij}=0$ for any $i\in C, j\in \bar{C}$. But this can not happen, because $M_{11}=1$ and $M_{1j}\geq 1+M_{jj}-1 = M_{jj}>0$ for any $j\in \bar{C}$. Contradiction! So $S$ is connected.

For the other direction we need to show that any connected subgraph $H=(S,F_S)\subseteq G=(V,E)$ has a corresponding matrix $H \rightleftharpoons M$, such that $M\in {\cal M}^*_a$ and $Q(M;\gamma)\succeq 0$ for some $a\in S$ and $\gamma>0$.

Let $M$ be defined as:
\begin{equation*}
  M_{ij} = \left\{    \begin{array}{cc}
                        1 & i\in S, \, j\in S \\
                        0 & \text{otherwise} \\
                      \end{array}
  \right.
\end{equation*}

This $M$ can be viewed as the adjacency matrix corresponding to a complete graph on the node set $S$. So it naturally involves a star graph centered at $a$, and satisfies the linear constraints of $\mathcal{M}^*_a$.

Furthermore, the sub-block corresponding to $S$, $A_S\circ M_S$, is exactly the adjacency matrix of $H$. Since $H=(S,F_S)$ is connected, the second smallest eigenvalue of $L_S(A_S\circ M_s)$ is strictly positive. Notice that on the sub-block, $M_S = \textbf{1}_k \textbf{1}_k'$.
Again by Finsler's Lemma, this means that there exists a $\gamma>0$, such that the LMI holds on the sub-block:
\begin{equation*}
  L_S(A_S \circ M_S) - \gamma L(M_S) \succeq 0
\end{equation*}
%
\end{proof}


{\bf Proof of Theorem \ref{thm:C_a_Phi}:}
\begin{proof}
For simplicity we provide a proof sketch for rectangle bands on a 2D lattice $G$.
We need to show that for a band $H=(S,F_S)$ belonging to ${\cal C}_{a,\Phi}$, there exists a binary matrix $M\rightleftharpoons H$ such that $L(A\circ M) - \gamma L(M) \succeq 0$, where $\gamma$ depends only on $\Phi$.

Construct the matrix $M$ as follows:
\begin{equation*}
  M_{ii}=\left\{   \begin{array}{cc}
                     1 & i\in S \\
                     0 & otherwise \\
                   \end{array}
    \right. , \,\,\, M_{ij}=\left\{   \begin{array}{cc}
                                        1 & (i,j)\in E_S \,\, \text{or} \,\, i=a \,\, \text{or} j=a \\
                                        0 & otherwise \\
                                      \end{array}
     \right.
\end{equation*}
Apparently $H \rightleftharpoons M$, and $M \in {\cal M}^*_a$.
W.l.o.g. assume $a=1$, and $S=\{1,2,...,k\}$. We only need to consider the first $k\times k$ sub-block of $Q(M;\gamma)$, denoted by $Q_S(M_S;\gamma)=L(A_S\circ M_S) - \gamma L(M_S)$.
Notice $L(A_S\circ M_S)$ is exactly the unnormalized Laplacian matrix of $H=(S,F_S)$, and $L(M_S)$ is the Laplacian of the union graph of $H$ and $H_{star}$, where $H_{star}$ denote the star graph centered at node $a$.

Let $M_S = A_S\circ M_S + M_\Delta$. $M_\Delta$ is the adjacency matrix of a graph $H_\Delta$, where $H_\Delta$ is obtained from $H_{star}$ by removing those edges connected with the anchor. We rewrite the required inequality:
\begin{equation*}
  Q_S(M_S;\gamma)=L(A_S\circ M_S) - \gamma L(M_S) = (1-\gamma)L(A_S\circ M_S) - \gamma L(M_\Delta) \succeq 0
\end{equation*}
Since $H_\Delta$ is obtained from $H_{star}$ by removing edges, we have $L(M_{star}) \succeq L(M_\Delta)$. We will show $\gamma = O(1/k) < 1/2$, which implies $\frac{\gamma}{1-\gamma}<2\gamma$.
Therefore it suffices to show:
\begin{equation*}
  L(A_S\circ M_S) - 2\gamma L(M_{star}) \succeq 0.
\end{equation*}

The rest part follows from Lemma \ref{lem:lattice_tree}, which characterizes the value of $\gamma$ for the above LMI to hold. Proof is done.
\end{proof}


\begin{lem}\label{lem:lattice_tree}
Let $G=(V,E)$ denote a $k$-node rectangle band with width $a$ and length $b$ on the 2D lattice, i.e. $ab=k$. Let $L$ be the graph Laplacian matrix corresponding to the rectangle lattice, and $L_{star}$ be the graph Laplacian of the star graph with the same node set, centered at the bottom-left node. Then the following inequality holds for $\gamma = \frac{\Phi^2}{4\log(k\Phi)}$:
\begin{equation*}
  L - \gamma L_{star} \succeq 0
\end{equation*}
\end{lem}

\begin{proof}
Assume the anchor node is node 1. It is equivalent to show that for any $f\in \mathbb{R}^k$,
\begin{equation*}
  f' L_{star} f = \sum_{i \geq 2} (f_1 - f_i)^2 \leq \frac{1}{\gamma} f' L f = \frac{1}{\gamma} \sum_{(i,j)\in E} (f_i - f_j)^2
\end{equation*}

We first investigate a simple case where $a=1$, i.e. $G$ is a $k$-node line graph. In this scenario $\phi(G)=2/k$.
We use Cauchy-Schwartz inequality to bound each $(f_1-f_i)^2$ using the edges on the path from node 1 to $i$:
\begin{equation*}
  (f_1-f_i)^2 = \left( \sum_{j=1}^{i-1}(f_j - f_{j+1}) \right)^2 \leq (i-1)\sum_{j=1}^{i-1} (f_j - f_{j+1})^2
\end{equation*}
Summing over all $(f_1-f_i)^2$, we have:
\begin{eqnarray*}
  && \sum_{i=2}^k (f_1-f_i)^2 \\
  &\leq& \sum_{i=2}^k \left[ (i-1)\sum_{j=1}^{i-1} (f_j - f_{j+1})^2 \right] \\
   &=& \left( \sum_{i=1}^{k-1} i \right) (f_1 - f_2)^2 + \left( \sum_{i=2}^{k-1} i \right) (f_2 - f_3)^2 + ... + (k-1) (f_{k-1} - f_k)^2  \\
   &\leq& \frac{k^2}{2} \sum_{j=1}^{k-1} (f_j - f_{j+1})^2
\end{eqnarray*}
Therefore the inequality for line graph holds.

Now w.l.o.g. assume $a\leq b$ and $a = 2^p$. We first show that to cover the $a^2/2$ nodes in the lower triangle, $\gamma=O(p2^p)=O(a^2\log a)$ is enough.
The strategy is similar: construct paths from anchor to each node, and apply Cauchy-Schwartz inequality to make use of edges on these paths.
Two tricks need to be mentioned: \\
(1) Paths need to be constructed very carefully so that each edge of $G$ is not used too often; \\
(2) It is inevitable that some edges will be used much more frequently than others, for example, the edges coming out of anchor. A weighted Cauchy-Schwartz should therefore be applied to alleviate this effect.

\begin{figure}[h]
\centering
\includegraphics[width=.96\textwidth]{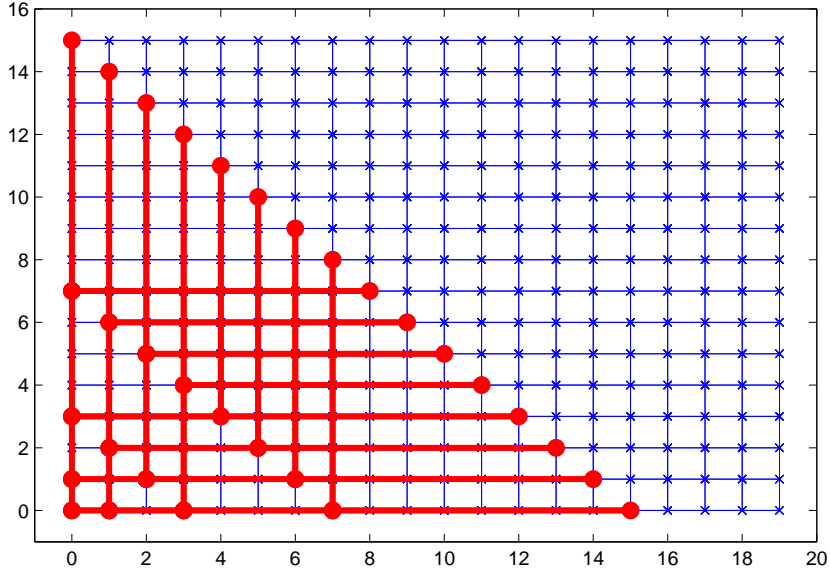}
\caption{ Paths constructed to cover each node from anchor. \label{fig:lattice_tree} }
\end{figure}

Let each node be indexed by its coordinates, $(0,0)$ is the anchor node. To help understand the construction, we introduce several notations. A node $v=(x,y)$ is ``critical'' if $x+y=2^q-1$ for some integer $q$, as marked by red solid circles in Fig.\ref{fig:lattice_tree}.
Let $\mathbb{C}_q = \{ v=(x,y) | x+y=2^q - 1 \}$ denote the collection of nodes on the $q$-th ``boundary''. Anchor node $v_0=(0,0)$ is the only node in $\mathbb{C}_0$, and the outer most boundary is $\mathbb{C}_p$. Apparently $|\mathbb{C}_q| = 2^q$.

We build a complete balanced binary tree based on all critical nodes with tree edges $(v_i, v_{i+1})$, where $v_i\in \mathbb{C}_i$ denotes a critical node in $\mathbb{C}_i$.
We note down several observations for paths from anchor to each $v_p\in\mathbb{C}_p$: \\
(1) There is a unique path starting from anchor $v_0\in\mathbb{C}_0$ to each $v_p\in\mathbb{C}_p$, passing through critical nodes $v_i\in\mathbb{C}_i$, for $i=0,1,...,p$. \\
(2) Such a path, denoted by $v_0\rightarrow v_1\rightarrow ... \rightarrow v_p$ where $v_i\in\mathbb{C}_i$, is composed of $p$ tree edges, $(v_i,v_{i+1})$ for $i=0,1,...,p-1$, with $|(v_i,v_{i+1})|=2^i$. \\
(3) For any two such paths, after they split at some node, they will never share any graph edges.

Now consider a path from $v_0$ to some $v_p\in \mathbb{C}_p$, $v_0\rightarrow v_1\rightarrow ... \rightarrow v_p$. We use weighted Cauchy-Schwartz inequality to bound this path with graph edges:
\begin{eqnarray*}
    && \left( f_{v_0} - f_{v_p} \right)^2 \\
   &=& \left( \sum_{i=0}^{p-1} (f_{v_i} - f_{v_{i+1}}) \right)^2  \\
   &=& \left( (f_{v_0}-f_{v_1}) + \sum_{(i,j)\in (v_1, v_2)}(f_i - f_j) + ... + \sum_{(i,j)\in (v_{p-1}, v_p)}(f_i - f_j) \right)^2 \\
   &\leq& (1\times 2^{p-1} + 2\times 2^{p-2} + ... + 2^{p-1}\times 1) \\
   &&   \cdot  \left( \frac{(f_{v_0}-f_{v_1})^2}{2^{p-1}} + \frac{\sum_{(i,j)\in (v_1, v_2)}(f_i - f_j)^2}{2^{p-2}} + ... + \frac{\sum_{(i,j)\in (v_{p-1}, v_p)}(f_i - f_j)^2}{1} \right) \\
   &=& p \left( (f_{v_0}-f_{v_1})^2 + 2 \sum_{(i,j)\in (v_1, v_2)}(f_i - f_j)^2 + ... + 2^{p-1} \sum_{(i,j)\in (v_{p-1}, v_p)}(f_i - f_j)^2  \right)
\end{eqnarray*}
The intuitive idea is that the graph edges composing tree edges closer to the anchor, i.e. $(i,j)\in (v_l,v_{l+1})$ for small $l$ where $v_l\in\mathbb{C}_l$, will be passed through many more times than those composing tree edges far away from the anchor.
So when applying weighted Cauchy-Schwartz inequality, a larger denominator is imposed on $(f_i - f_j)^2$ for those $(i,j)\in (v_l,v_{l+1})$ for small $l$.
For example, for the most frequently used edge $(v_0,v_1)$, a penalty of $2^{p-1}$ is imposed on these edges (2 such edges, ((0,0),(0,1)) and ((0,0),(1,0))), while for those graph edges composing $(v_{p-1},v_p)$, only a constant is put in the denominator.

Next we need to figure out the frequency that each graph edge is used for covering all the nodes. By induction it is not hard to observe that the graph edges on the tree edge $(i,j)\in (v_l,v_{l+1})$ will be passed by at most $2^{2p-1-l}$ paths.
Take the graph of Fig.\ref{fig:lattice_tree} as an example. Each path is of the form $v_0\rightarrow ... \rightarrow v_4$, $v_i\in\mathbb{C}_i$.
The edges on $(v_3,v_4)$ are used at most 8 times, eg. $\left((7,0), (8,0)\right)$. We have $8<16=2^{2p-1-3}$.
The edges on $(v_2,v_3)$ are used at most $8\times 2 + 4 = 20$ times, eg. $\left((3,0), (4,0)\right)$. $20<32=2^{2p-1-2}$.
The edges on $(v_1,v_2)$ are used at most $20\times 2 + 2 = 42$ times, eg. $\left((1,0), (2,0)\right)$. $42<64=2^{2p-1-1}$.
The top-most edges, $\left((0,0),(1,0)\right)$ and $\left((0,0),(0,1)\right)$, are used $42\times 2 + 1 = 85$ times. $85<128=2^{2p-1-0}$.

So summing over all paths from anchor to all nodes within the lower triangle $T$:
\begin{eqnarray*}
  && \sum_{v\in T} (f_{v_0} - f_v)^2  \\
  &\leq&  p \sum_{v_0\rightarrow ... \rightarrow v_p\in \mathbb{C}_p} \left( 2^{2p-1} (f_{v_0}-f_{v_1})^2 + ... + 2^{2p-1)} \sum_{(i,j)\in (v_{p-1}, v_p)}(f_i - f_j)^2 \right)  \\
   &\leq& p 2^{2p-1} \sum_{(i,j)\in E} (f_i - f_j)^2
\end{eqnarray*}

Note that $p2^{2p-1}=a^2 \log a/2$. So:
\begin{equation*}
  \gamma=\frac{2}{a^2\log a}
\end{equation*}
is enough to cover all nodes in the lower triangle of an $a\times b$ rectangle lattice as in Fig.(\ref{fig:lattice_tree}).

To cover the rest nodes, i.e. blue nodes in Fig.\ref{fig:lattice_tree}, we build paths that horizontally extend from the outer-most boundary nodes $v_p\in\mathbb{C}_p$. Let $v_{p'}$ denote the rightmost node extending horizontally from $v_p\in\mathbb{C}_p$. Similarly we use weighted Cauchy-Schwartz inequality to bound the path: $v_0\rightarrow ... \rightarrow v_p \rightarrow v_{p'}$:
\begin{eqnarray*}
  && \left( f_{v_0} - f_{v_{p'}} \right)^2 \\
  &=& \left( (f_{v_0}-f_{v_1}) + ... + \sum_{(i,j)\in (v_{p-1}, v_p)}(f_i - f_j) + \sum_{(i,j)\in (v_{p}, v_{p'})}(f_i - f_j) \right)^2 \\
   &\leq& \left(1\times 2^{p-1} + 2\times 2^{p-2} + ... + 2^{p-1}\times 1 + b \times 1 \right) \\
   &&   \cdot  \left( \frac{(f_{v_0}-f_{v_1})^2}{2^{p-1}} + ... + \frac{\sum_{(i,j)\in (v_{p-1}, v_p)}(f_i - f_j)^2}{1} + \frac{\sum_{(i,j)\in (v_{p}, v_{p'})}(f_i - f_j)^2}{1} \right) \\
   &=& \left( p 2^{p-1} + b \right) \left( \sum_{l=0}^{p-1} \frac{\sum_{(i,j)\in (v_{l}, v_{l+1})}(f_i - f_j)^2}{2^{p-1-l}} + \sum_{(i,j)\in (v_{p}, v_{p'})}(f_i - f_j)^2 \right)
\end{eqnarray*}

It is easy to observe that to cover these extended nodes, the graph edges $(i,j)\in (v_l,v_{l+1})$ are passed through $b 2^{p-1-l}$ times for $l=0,1,...,p-1$, and $b$ times for those extended edges $(i,j)\in (v_p, v_{p'})$. Now totally we have:
\begin{equation*}
  \sum_{v} (f_{v_0} - f_v)^2 \leq \left( p 2^{2p-1} + b (p 2^{p-1} + b) \right) \sum_{(i,j)\in E} (f_i - f_j)^2
\end{equation*}
Plugging in $2^p=a$, $a\leq b$ and $ab=k$, we have:
\begin{eqnarray*}
  \sum_{v} (f_{v_0} - f_v)^2  &\leq& \left( ab\log a + b^2 \right) \sum_{(i,j)\in E} (f_i - f_j)^2 \\
   &\leq& \max \left( 2k \log \frac{k}{b}, 2b^2 \right) \sum_{(i,j)\in E} (f_i - f_j)^2
\end{eqnarray*}

Note that $\Phi = \frac{a}{k/2} = \frac{2}{b}$. Replace $b$ with $\Phi$, the proof is done.

We list two extreme examples for demonstration.
For the thinnest line graph where $a=1,b=k$ and $\Phi = 2/k$, $\gamma=\frac{1}{2 k^2} = \Phi^2 / 8$ is sufficient to have: $L - \gamma L_{star} \succeq 0$. For the other extreme case where the graph is a square lattice with $a=b=\sqrt{k}$, $\Phi = 2/\sqrt{k}$, $\gamma = \frac{1}{k\log k\Phi} = \Phi^2 / 4\log (k\Phi)$ is required for the LMI to hold.
Note that $\Phi$ is between $O(1/\sqrt{k})$ and $\Omega(k)$. So at least the smaller $\gamma = \Theta(\Phi^2 / \log k )$ can make the LMI hold.
Proof is done.

\end{proof}


For future use we present the explicit form of the dual problem to a primal problem that has constraints $M\in \mathcal{C}_{LMI}(a,\gamma)$.
Interestingly, the dual problem corresponds to finding an embedding of all nodes in a 1D Euclidean space, such that certain constraints at each node and edge of the graph hold.

\begin{lem}\label{lem:primal_dual}
Given $G=(V,E)$ with adjacency matrix $A$, let $y_i$ denote the variable associated with node $i\in V$.
Assume w.l.o.g. the anchor is node 1.
Consider the following SDP problem, where the constraints are exactly those of $M\in \mathcal{C}_{LMI}(1,\gamma)$:
\begin{eqnarray}\label{eq:primal}
  \max: && \sum_{i}y_i M_{ii} \\
 \nonumber
   s.t. && Q(M;\gamma) = L(A\circ M) - \gamma L(M) \succeq 0  \\
 \nonumber
       && M_{ij} \geq 0, \,\,\, \forall 2\leq i < j \\
 \nonumber
       && 1-M_{ii} \geq 0, \,\, \forall 2\leq i  \\
 \nonumber
       && M_{ii} - M_{ij} \geq 0, \,\, \forall 2\leq i<j \\
 \nonumber
       && M_{jj} - M_{ij} \geq 0, \,\, \forall 2\leq i<j
\end{eqnarray}

Then the corresponding dual problem has the following form:
\begin{eqnarray}\label{eq:dual}
  \min: && \,\, y_1 + \sum_{i\geq 2}\rho_i  \\
\nonumber
  s.t.  && \,\, y_i + \left( 1 - \gamma \right) z_i^2 + \sum_{2\leq j\neq i,(i,j)\in E}\alpha_{ij} + \alpha_i = \rho_i,\,\,\forall i\geq 2,\,(1,i)\in E \\
\nonumber
        && \,\, y_i - \gamma z_i^2 + \sum_{2\leq j\neq i,(i,j)\in E}\alpha_{ij} + \alpha_i = \rho_i,\,\,\forall i\geq 2,\,(1,i)\notin E \\
\nonumber
        && \,\, \left( 1 - \gamma \right) (z_i - z_j)^2 \leq \alpha_{ij} + \alpha_{ji}, \,\, \forall 2\leq i<j,\, (i,j)\in E  \\
\nonumber
        && \,\, \rho_i\geq 0, \alpha_{ij}\geq 0, \alpha_i\geq 0, z_i\geq 0
\end{eqnarray}
where $z_i$, a scalar dual variable, is the embedding coordinate of node $i\geq 2$; the rest dual variables include $\alpha_i,\rho_i,\forall i\geq 2$ and $\alpha_{ij},\forall (i,j)\in E$.
\end{lem}
\begin{proof}

The explicit Lagrangian of Eq.(\ref{eq:primal}) is:
\begin{eqnarray}\label{eq:lagrangian}
  L &=& y_1 + \sum_{i\geq 2}M_{ii}y_i + \langle Q,G \rangle + \sum\sum_{2\leq i<j}\mu_{ij}M_{ij} + \sum_{i\geq 2}\rho_i\left( 1-M_{ii} \right) \\
\nonumber
    &&  + \sum\sum_{2\leq i<j} \alpha_{ij}\left( M_{ii}-M_{ij} \right) + \sum\sum_{2\leq j<i} \alpha_{ij}\left( M_{ii}-M_{ji} \right)
\end{eqnarray}
where $G\succeq 0,\mu_{ij}\geq 0,\rho_i\geq 0,\alpha_{ij}\geq 0$ are lagrange multipliers. Notice the symmetric matrix $Q$ can be decomposed into the following form:
\begin{equation}\nonumber
  Q(M;\gamma) = L(A\circ M) - \gamma L(M) = \sum\sum_{i < j}\left( \textbf{1}_{(i,j)} - \gamma \right) M_{ij} \left( e_{ii} + e_{jj} - e_{ij} - e_{ji} \right)
\end{equation}
where $1_{(i,j)}$ is the indicator of $(i,j)\in E$, $e_{ij}$ denotes the matrix with value 1 at $(i,j)$ and 0 elsewhere.
Plugging in $M_{1i}=M_{ii}$, we have:
\begin{eqnarray*}
  \langle Q,G \rangle &=& \sum_{i\geq 2} \left( \textbf{1}_{(1,i)} - \gamma \right) M_{ii} \left( G_{11} + G_{ii} - 2G_{1i} \right) \\
   && + \sum\sum_{2\leq i < j}\left( \textbf{1}_{(i,j)} - \gamma \right) M_{ij} \left( G_{ii} + G_{jj} - 2G_{ij} \right)
\end{eqnarray*}

Taking derivatives w.r.t. $M_{ii}$ and $M_{ij}$ respectively, the dual problem is:
\begin{eqnarray}\label{eq:dual_mu1}
  \min: && \,\, y_1 + \sum_{i\geq 2}\rho_i  \\
\nonumber
  s.t.  && \,\, y_i + \left( 1_{(1,i)} - \gamma \right)G_{(1i)} + \sum_{2\leq j\neq i}\alpha_{ij} = \rho_i ,\,\,\forall i\geq 2 \\
\nonumber
   && \,\, \left( 1_{(i,j)} - \gamma \right)G_{(ij)} + \mu_{ij} - \alpha_{ij} - \alpha_{ji} = 0, \,\, \forall 2\leq i<j  \\
\nonumber
   && \,\, G\succeq 0,  \mu_{ij}\geq 0,  \rho_i\geq 0,  \alpha_{ij}\geq 0
\end{eqnarray}
where $G_{(ij)}=G_{ii}+G_{ii}-2G_{ij}$.

Since $G$ is symmetric and PSD, we have $G=VV'$ such that $G_{ij}=v_i' v_j$.
$v_i\in \mathbb{R}^n$ can be viewed as the embedding of node $i$ in the $n$-dimensional Euclidean space.
$G_{(ij)}=||v_i - v_j||^2$ is simply the squared distance between the embeddings of node $i$ and $j$.
We write constraints separately based on indicators:
\begin{eqnarray}\label{eq:dual_mu2}
  \min: && \,\, y_1  + \sum_{i\geq 2}\rho_i  \\
\nonumber
  s.t.  && \,\, y_i  + \left( 1 - \gamma \right)||v_i - v_1||^2 + \sum_{2\leq j\neq i}\alpha_{ij} = \rho_i,\,\,\forall i\geq 2,\,(1,i)\in E \\
\nonumber
        && \,\, y_i  - \gamma ||v_i - v_1||^2 + \sum_{2\leq j\neq i}\alpha_{ij} = \rho_i,\,\,\forall i\geq 2,\,(1,i)\notin E \\
\nonumber
        && \,\, \left( 1 - \gamma \right)|| v_i - v_j ||^2 + \mu_{ij} - \alpha_{ij} - \alpha_{ji} = 0, \,\, \forall 2\leq i<j,\, (i,j)\in E  \\
\nonumber
        && \,\, - \gamma || v_i - v_j ||^2 + \mu_{ij} - \alpha_{ij} - \alpha_{ji} = 0, \,\, \forall 2\leq i<j,\, (i,j)\notin E  \\
\nonumber
        && \,\,\mu_{ij}\geq 0, \rho_i\geq 0, \alpha_{ij}\geq 0
\end{eqnarray}

We further simplify this dual formulation.
Notice that for constraints of $(i,j)\notin E$, $\mu_{ij}\geq 0$ is an independent and completely free variable that can always make such a constraint hold. So we can drop these redundant constraints. For edge constraints of $(i,j)\in E$, we replace $\mu_{ij}$ with inequalities.
For node constraints of node $i$, we split out those $\alpha_{ij}$ with $(i,j)\notin E$ which are independent and combine them into a new variable $\alpha_i\geq 0$.
Also note that the embedding of anchor, $v_1$, is completely free variable, which we can fix w.l.o.g. at 0.
The dual problem is simplified as follows:
\begin{eqnarray}\label{eq:dual_v}
  \min: && \,\, y_1 + \sum_{i\geq 2}\rho_i  \\
\nonumber
  s.t.  && \,\, y_i + \left( 1 - \gamma \right)|| v_i ||^2 + \sum_{2\leq j\neq i,(i,j)\in E}\alpha_{ij} + \alpha_i = \rho_i,\,\,\forall i\geq 2,\,(1,i)\in E \\
\nonumber
        && \,\, y_i - \gamma || v_i ||^2 + \sum_{2\leq j\neq i,(i,j)\in E}\alpha_{ij} + \alpha_i = \rho_i,\,\,\forall i\geq 2,\,(1,i)\notin E \\
\nonumber
        && \,\, \left( 1 - \gamma \right)|| v_i - v_j ||^2 \leq \alpha_{ij} + \alpha_{ji}, \,\, \forall 2\leq i<j,\, (i,j)\in E  \\
\nonumber
        && \,\, \rho_i\geq 0, \alpha_{ij}\geq 0, \alpha_i\geq 0
\end{eqnarray}

Note the constraints have been divided into 3 categories: node constraints of those nodes directly linking to the anchor node, node constraints of the rest nodes, and edge constraints of edges among all nodes except the anchor.

The key observation is that each embedding vector $v_i$ only appears in node constraints with its length $||v_i||$, while only distances between embeddings exist in edge constraints, which are all inequalities.
We perform several operations on $v_i$ while maintaining dual feasibility.
The first step is to fold all $v_i$ into a fixed quadrant so that $||v_i||$ remains unchanged while $||v_i-v_j||$ either remains unchanged or is decreased.
This can be done by first fixing a Euclidean coordinate system, with $n$ hyperplanes intersecting at 0 and pairwise perpendicular. Then for each such hyperplane that partitions the whole space into two half-spaces, we fold all $v_i$ in the ``left'' half-space to the ``right'' half-space axis-symmetrically. It is obvious that this folding operation maintains $||v_i||$ for all $i$ and $||v_i-v_j||$ for those $i,j$ in the same half-space. The rest $i,j$, $||v_i-v_j||$ are only decreased due to Pythagoras theorem.
After folding for all these hyperplanes, all $v_i$ now locate in the same quadrant such that $v_i'v_j \geq 0,\,\forall i,j$, i.e. angles between $v_i$ and $v_j$ are smaller than $\pi/2$. Yet all node and edge constraints are still satisfied.

The second step is mapping all $v_i$ onto one single direction:
\begin{equation*}
  v_i \in \mathbb{R}^n \mapsto z_i \in \mathbb{R}^+ : \,\,\, z_i = ||v_i||
\end{equation*}
By definition all node constraints are satisfied. Again by Pythagoras and the $\pi/2$ condition, $||v_i-v_j||$ is decreased so that edge constraints are satisfied.
Therefore the dual problem Eq.(\ref{eq:dual_v}) can be reduced to the equivalent Eq.(\ref{eq:dual}). Proof is done.

\end{proof}

To prove the main theorem, we need the following lemma.

\begin{lem}\label{lem:max_trace}
On a graph with maximum degree $D$, consider the following max-trace problem:
\begin{eqnarray}\label{eq:max_trace}
  \max : && \, tr(M) \\
\nonumber
  s.t. && L(A\circ M) - \gamma L(M) \succeq 0 \\
\nonumber
   && M_{ij}=M_{ji}, \, M_{11}=1, \, M_{1i}=M_{ii} \\
\nonumber
   && 0\leq M_{ij}\leq M_{ii}, M_{jj} \leq 1
\end{eqnarray}
Let $M^*=M^*(\gamma)$ be the optimal solution to this problem. Then $M^*$ has the following properties:
\begin{enumerate}
  \item $tr(M^*) \leq D/\gamma$, where $D$ is the max degree of the graph.
  \item The node set $V_0 = \{i: \, M^*_{ii}=1\}$, including the anchor, form a connected sub-graph.
  \item The 1-hop outer layer, $V_1 = \{i:\, (i,j)\in E, j\in V_0, i\notin V_0\}$, satisfy: $0\leq M^*_{ii} < 1$.
  \item The rest nodes are: $M^*_{ii}=0$.
\end{enumerate}
\end{lem}

\noindent
\textbf{Remark:} \\
This lemma is just saying that the solution $M^*$ to the max-trace problem has a nested structure centered at the anchor. The interior of the support of $diag(M^*)$ have value $M^*_{ii}=1$, the boundary $0\leq M^*_{ii}<1$, and the rest nodes have $M^*_{ii}=0$.
We conjecture that $M^*$ always has a ``fattest'' shape.
At least by Theorem \ref{thm:C_a_Phi} $M^*$ contains a square of size $\Theta(k)$ if $\gamma=\Theta(\frac{1}{k\log k})$.
Fig.\ref{fig:max_trace_demo} shows two solutions of the max-trace problem with different values of $\gamma$.
Intuitively, smaller $\gamma$ allows the search to extend farther away than larger $\gamma$.

\begin{figure*}[htb]
\begin{centering}
\begin{minipage}[t]{.48\textwidth}
\includegraphics[width = 1\textwidth]{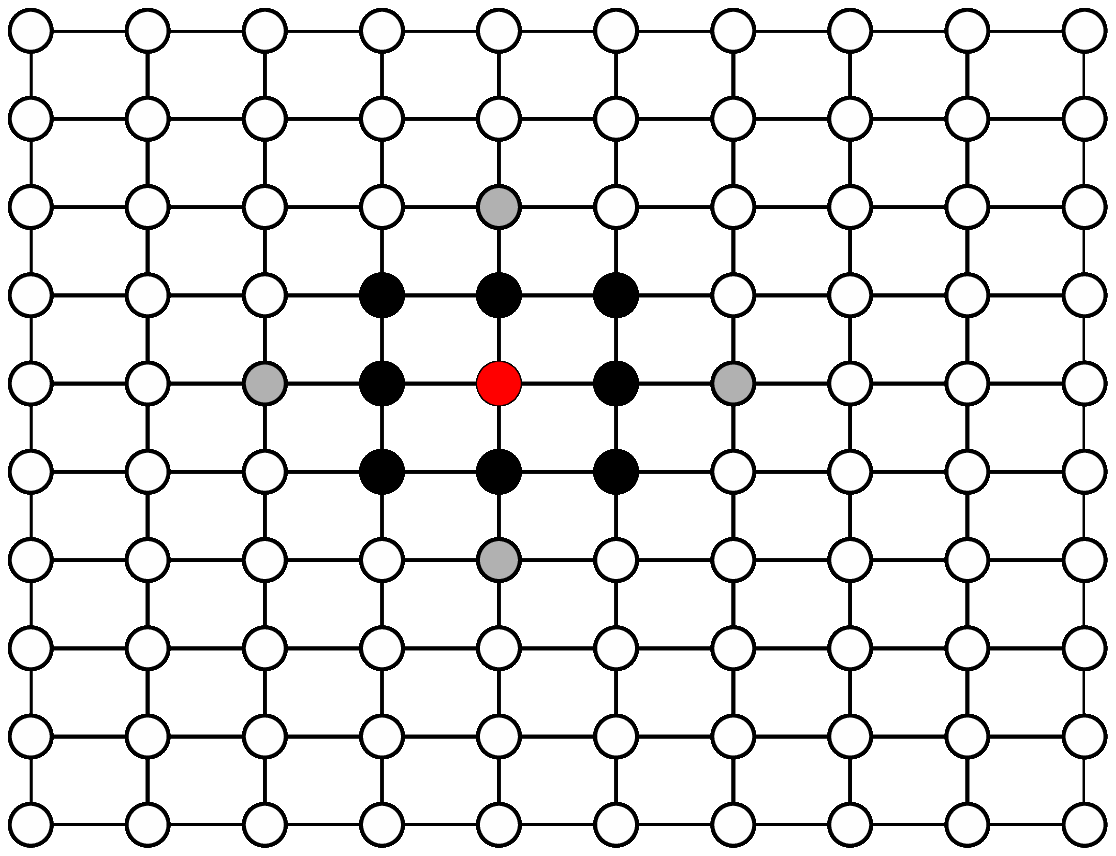}
\makebox[7 cm]{(a) $\gamma=0.3$}
\end{minipage}
\begin{minipage}[t]{.48\textwidth}
\includegraphics[width = 1\textwidth]{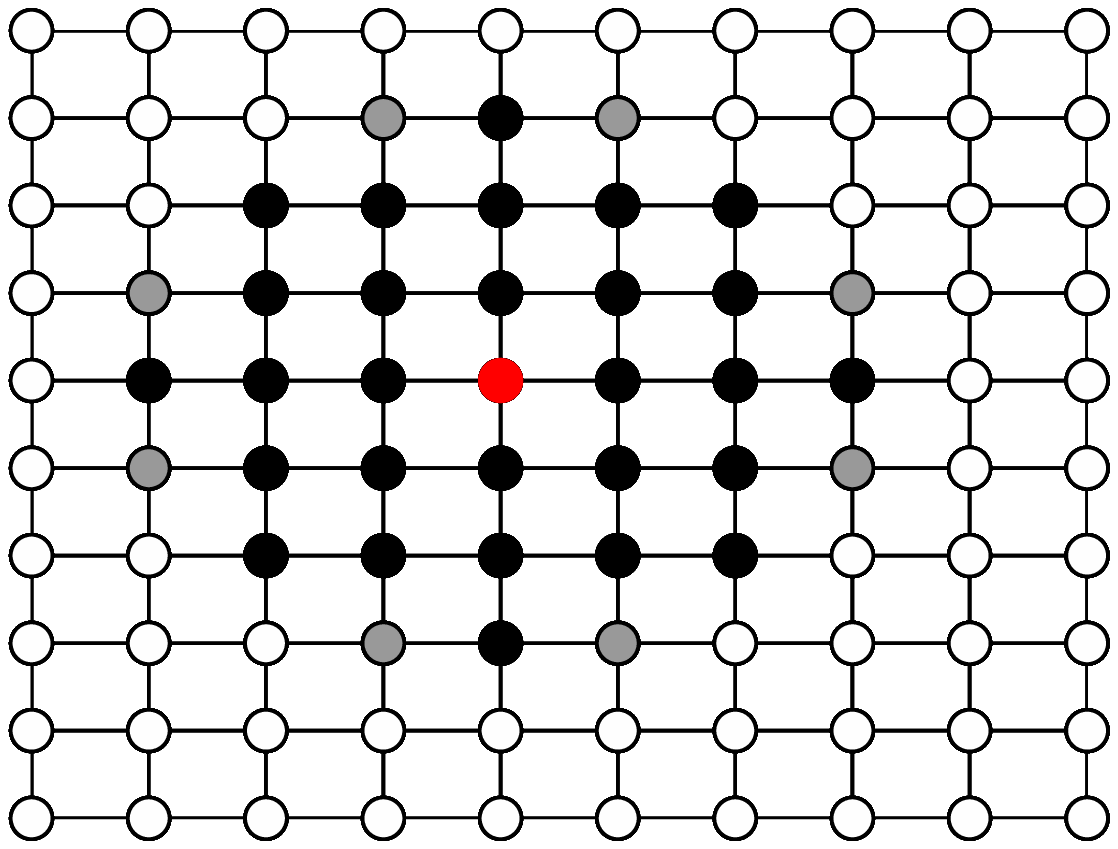}
\makebox[7 cm]{(b) $\gamma=0.08$}
\end{minipage}
\caption{ Optimal solution $M^*$ of max-trace problem, with large / small values of $\gamma$. Values of $M^*_{ii}$ are illustrated through grey-scale. Rped node is the anchor. \label{fig:max_trace_demo} }
\end{centering}
\end{figure*}

\begin{proof}
This is the problem of Eq.(\ref{eq:primal}) with all $y_i=1$. According to Eq.(\ref{eq:dual}), the corresponding dual problem is:
\begin{eqnarray*}
  \min: && \,\, 1 + \sum_{i\geq 2}\rho_i  \\
\nonumber
  s.t.  && \,\, 1 + \left( 1 - \gamma \right)z_i^2 + \sum_{2\leq j\neq i,(i,j)\in E}\alpha_{ij} + \alpha_i = \rho_i,\,\,\forall i\geq 2,\,(1,i)\in E \\
\nonumber
        && \,\, 1 - \gamma z_i^2 + \sum_{2\leq j\neq i,(i,j)\in E}\alpha_{ij} + \alpha_i = \rho_i,\,\,\forall i\geq 2,\,(1,i)\notin E \\
\nonumber
        && \,\, \left( 1 - \gamma \right) (z_i - z_j)^2 \leq \alpha_{ij} + \alpha_{ji}, \,\, \forall 2\leq i<j,\, (i,j)\in E  \\
\nonumber
        && \,\, \rho_i\geq 0, \alpha_{ij}\geq 0, \alpha_i\geq 0, z_i\geq 0
\end{eqnarray*}
We show (1) by constructing a simple dual solution to yield an upper bound on the max-trace problem.
Let $z_i=z, \forall i\geq 2$, so that all edge constraints automatically hold.
Let $z^2=1/\gamma$, $\alpha_{ij}=\alpha_i=0$, so that $\rho_i=0,\forall i\geq 2, (1,i)\notin E$.
The cost of this dual feasible solution, thus an upper bound on $tr(M^*)$, is:
\begin{equation*}
  tr(M^*) \leq \sum_{i:(1,i)\in E} \left( 1 + (1-\gamma)z^2 \right) \leq D / \gamma
\end{equation*}

The intuition is that $z_i$ increases as $i$ goes farther away from the anchor, until $\gamma z_i^2 \geq y_i = 1$ for all nodes $i$ outside some closed layer $B$ which contains the anchor. This layer corresponding to the above trivial solution is simply the set of 1-hop neighbors of anchor: $B=\{i: (1,i)\in E\}$. But this dual feasible solution increases $z_i$ too fast (in one step), thus pays too much price at $\rho_i$ for these direct neighbors.

Let $V_0$ be the set of nodes with $M^*_{ii}=1$ and connected to the anchor node 1.
Let $V_1$ be the 1-hop outer layer of $V_0$, and $V_2$ the 1-hop outer layer of $V_1$.
Since strong duality holds, by complementary slackness, the optimal dual variables have: $\rho_i=0, \forall i\in V_1$.
We create slackness for all edges between $V_1$ and $V_2$, which correspond to the original dual variables $\mu_{ij}$ back in Eq.(\ref{eq:dual_mu2}).
Again by complementary slackness, if $\mu_{ij}>0$, then the primal $M_{ij}=0$. We have disconnected nodes in $V_0$ from outside $V_1$. By Theorem \ref{thm:connectivity} the support of $diag(M)$ is connected. So $M_{ii}=0$ for those nodes outside $V_1$.

To create this slackness for edges between $V_1$ and $V_2$, consider a modified primal objective:
\begin{eqnarray}\label{eq:max_trace_modified}
  \max : && \, \sum_{i\in V_0 \cup V_1} M_{ii} + (1 - \epsilon ) \sum_{ i\notin V_0 \cup V_1} M_{ii} \\
\nonumber
  s.t. && M \in \mathcal{C}_{LMI}(1,\gamma)
\end{eqnarray}
The optimal dual solution to the max-trace problem is also feasible for this modified problem, which gives the same dual cost.
Now outside $V_1$:
\begin{eqnarray*}
  && \,\, \left( 1 - \gamma \right) (z_i - z_j)^2 \leq \alpha_{ji} + \alpha_{ij}, \,\,\,\,\,\,\,\,\,\,\,\,\,\,\,\,\,\, \forall j\in V_1, i\in V_2 \, (j,i)\in E  \\
  && \,\, 1 - \epsilon - \gamma z_i^2 + \sum_{2\leq j\neq i,(i,j)\in E}\alpha_{ij} + \alpha_i = \rho_i,\,\, \forall i \in V_2
\end{eqnarray*}
Leaving other dual variables unchanged, we can distribute $\epsilon$ uniformly on those $\alpha_{ij}, \, i\in V_2, j\in V_1$ to create the slackness we want on edges $(j,i)$.
Proof is done.

\end{proof}
The proofs of main theorems for Poisson and Gaussian models follow similar lines. Here we only elaborate on the Gaussian case.

\textbf{ Proof for Gaussian model: }
\begin{proof}
The proof consists of 2 parts:
\begin{itemize}
  \item Inseparability: This part generalizes the results of \cite{Castro11} in terms of the internal conductance parameter $\Phi$ rather than the length and width used in \cite{Castro11}. This is shown in Lemma \ref{lem_inseparability}.
  \item Separability: This part itself can be divided into two steps.
  \begin{enumerate}
    \item We first show under $H_0$ the optimal value of the test is upper bounded by using a modified version of $M^*$, the optimal solution to the max-trace problem. This is shown in Lemma \ref{lem:H0_upper}.
    \item We then show that under $H_1$, the feasible solution $M^*$ to the max-trace problem covers a large portion of the ground-truth cluster for our problem.
  \end{enumerate}
\end{itemize}
By Lemma \ref{lem:H0_upper} we have:
\begin{equation*}
  c^*|_{H_0} \leq N(0, tr(M^*)) + O \left( \sqrt{\frac{\log k}{\gamma}} \right)
\end{equation*}

For the $H_1$ case, for simplicity we consider a band $B$ of size $k$, with width $a$ and length $b$, $ab=k$. The corresponding conductance is $\Phi=\Theta(1/b)$.
Such a band must be contained in a square of size $b\times b$, i.e. $\Theta(1/\Phi^2)$.
On the other hand, for this band we choose $\gamma=\Theta(\Phi^2/\log k)$.
The $M^*$ of the max-trace problem with this $\gamma$ at least contains a square of size $\Theta(1/\Phi^2)$. Therefore by appropriately positioning the anchor, $M^*$ overlaps $B$ at least on $\Theta(k)$ nodes.
%
This means if we simply adopt $M^*$ as a primal feasible solution, we have:
\begin{equation*}
 c^*|_{H_1} \geq N(0,tr(M^*)) + \Theta(k) \mu
\end{equation*}

Note that $tr(M^*) = O(1/\gamma)$. To asymptotically separate $H_0$ and $H_1$, it suffices that:
\begin{equation*}
  tr(M^*) + O(\sqrt{1/\gamma}) + O(\sqrt{ \log k/\gamma })  \leq    tr(M^*) - O(\sqrt{1/\gamma}) + \Theta(k) \mu ,
\end{equation*}
where the terms $O( \sqrt{ 1/\gamma } )$ on both sides correspond to the standard deviation term.
Plugging in $\gamma = \Phi^2 / log(k)$, we have:
\begin{equation*}
  \mu = \Omega \left(  \frac{ \log k}{ k \Phi } \right)
\end{equation*}
When the anchor is unknown, applying the test for different anchors induces an additional $\sqrt{\log n}$ term due to union bound.
When the shape is unknown, the test sets $\gamma$ according to the smallest conductance, i.e. $\gamma = \Theta(1/k^2)$, to search for the thinnest shape with size $k$. In this case, the requirement on $\mu$, when agnostic to anchor and shape, is:
\begin{equation*}
  \mu = \Omega \left( \log k \sqrt{ \log n } \right)
\end{equation*}
Proof is done.
\end{proof}


\begin{lem}\label{lem_inseparability}
The two hypothesis $H_0$ and $H_1$ are asymptotically inseparable if:
\begin{equation*}
  \mu_n K_n \Phi_n \log(K_n) \rightarrow 0
\end{equation*}
\end{lem}
\begin{proof}
The collection of anomalous subgraphs with size $K_n$ and internal conductance $\Phi_n$ contains the bands of width $h_n$ and length $l_n$ defined in Theorem 3 of \cite{Castro11}. So the inseparability result there also holds for our case.
Roughly we have:
\begin{equation*}
  l_n h_n = K_n, \,\,\,\,\,\, \frac{h_n}{K_n} = \Phi_n
\end{equation*}

By Theorem 3 in \cite{Castro11}, $H_0$ and $H_1$ are asymptotically inseparable if: (ignoring the $\log\log()$ term)
\begin{equation*}
  \mu_n \sqrt{K_n} \left( \frac{l_n}{h_n} \right)^{-1/2} \log(l_n) \rightarrow 0
\end{equation*}

Substitute $l_n$ and $h_n$ using $K_n$ and $\Phi_n$, and note that $1/\Phi \geq \sqrt{K_n}$.
We get:
\begin{equation*}
  \mu_n K_n \Phi_n  \log(K_n) \rightarrow 0.
\end{equation*}
\end{proof}

\begin{lem}\label{lem:H0_upper}
Assume $x_i$ follows standard normal distribution for all nodes $i$.
The optimal cost of problem Eq.(\ref{eq:primal}) with signal $x_i$ for node $i$ is upper bounded by:
\begin{equation*}
  c^*|_{H_0} \leq \sum_{i}x_i M^*_{ii} + \Theta\left(\sqrt{\log\left(\frac{1}{\gamma}\right) /\gamma} \right)
\end{equation*}
where $M^*$ is the optimal solution to the max-trace problem with parameter $\gamma$.
\end{lem}

\begin{proof}
Let $y_i = 1  + x_i / N$, where $N$ is a normalization constant to be decided.
We show that for appropriately chosen $N$, the modified problem Eq.(\ref{eq:primal}) with signal $y_i$ has the optimal cost with some upper bound. We then recover the original problem by first subtracting $tr(M^*)$, following by multiplying $N$.

Write $y_i = (1 + x_{max}/N) - (x_{max} - x_i)/N = (1+x_{max}/N) - \eta_i$, where $x_{max} = \max_{i\in H^*(M^*)} \, |x_i|$, $\eta_i = (x_{max} - x_i)/N$. Note that $x_{max}$ scales as $\Theta(\sqrt{|H^*|})$ for i.i.d. standard normal random variables, where $H^*(M^*)$ is the resulting fat shape corresponding to the max-trace problem.
Note that $0\leq \eta_i \leq 2 x_{max}/N$ for $i\in H^*$.
Consider the dual solution of the max-trace problem. We know that for nodes $i\in V_0$ the dual variables $\rho_i > 0$. Let $\delta = \min_{i\in V_0} \rho_i > 0$, which is a constant depending only on $\gamma$.
Consider the following problem:
\begin{eqnarray}\label{eq:max_trace_stretched}
  \max : && \, (1+x_{max}/N) tr(M) \\
\nonumber
  s.t. && M \in \mathcal{C}_{LMI}(1,\gamma)
\end{eqnarray}
Since $x_{max}$ is just a constant, the optimal dual solution to this problem is just the $(1+x_{max}/N)$-stretched version of that of the max-trace problem. So $\min_{i\in V_0} \rho'_i = (1+x_{max}/N)\delta > \delta$.
Now choose $N$ sufficiently large such that
\begin{equation*}
  \eta_i \leq 2 x_{max}/N \leq \delta < \min_{i\in V_0} \rho'_i
\end{equation*}

We modify this dual solution of Eq.(\ref{eq:max_trace_stretched}) to build a dual feasible solution for:
\begin{eqnarray}\label{eq:max_yi_normalized}
  \max : && \, \sum_i y_i M_{ii} \\
\nonumber
  s.t. && M \in \mathcal{C}_{LMI}(1,\gamma)
\end{eqnarray}

Let $\tilde{c}$ denote the optimal cost. By Lemma \ref{lem:primal_dual} the corresponding dual problem is:
\begin{eqnarray}\label{eq:max_yi_normalized_dual}
  \min: && \,\, y_1 + \sum_{i\geq 2}\rho_i  \\
\nonumber
  s.t.  && \,\, 1 + \frac{x_{max}}{N} - \eta_i + \left( 1 - \gamma \right)z_i^2 + \sum_{2\leq j\neq i,(i,j)\in E}\alpha_{ij} + \alpha_i = \rho_i,\,\,\forall i\geq 2,\,(1,i)\in E \\
\nonumber
        && \,\, 1 + \frac{x_{max}}{N} - \eta_i - \gamma z_i^2 + \sum_{2\leq j\neq i,(i,j)\in E}\alpha_{ij} + \alpha_i = \rho_i,\,\,\forall i\geq 2,\,(1,i)\notin E \\
\nonumber
        && \,\, \left( 1 - \gamma \right) (z_i - z_j)^2 \leq \alpha_{ij} + \alpha_{ji}, \,\, \forall 2\leq i<j,\, (i,j)\in E  \\
\nonumber
        && \,\, \rho_i\geq 0, \alpha_{ij}\geq 0, \alpha_i\geq 0, z_i\geq 0
\end{eqnarray}
The only differences between Eq.(\ref{eq:max_yi_normalized_dual}) and the dual problem of Eq.(\ref{eq:max_trace_stretched}) are those $-\eta_i$ at node constraints.
Based on the dual optimal solution of Eq.(\ref{eq:max_trace_stretched}), we modify dual variables to build a dual feasible solution for Eq.(\ref{eq:max_yi_normalized_dual}).
Two cases need to be considered.
\begin{itemize}
  \item For nodes $i \in V_0$, simply let $\rho'_i = \rho_i - \eta_i$. Note that we still have dual feasibility: $\rho'_i \geq 0$ by construction of $N$.
  \item For nodes $i \in H^* - V_0$ where $\rho'_i = \rho_i = 0$, we increase the free variables, $\alpha'_i = \alpha_i + \eta_i$, to absorb the difference, while keeping $\rho'_i=0$ unchanged.
  \item For nodes outside $H^*$, since we know the size $k$, for ease of proof we simply zero out all $y_i$.
\end{itemize}

In this way we have built a dual feasible solution of Eq.(\ref{eq:max_yi_normalized_dual}).
The corresponding dual cost, thus an upper bound on the primal optimum of Eq.(\ref{eq:max_yi_normalized}) by weak duality, is:
\begin{eqnarray*}
  \tilde{c} &\leq& (1 + \frac{x_{max}}{N}) tr(M^*) - \sum_{i\in V_0} \eta_i  \\
   &=& tr(M^*) + \frac{x_{max}}{N} tr(M^*) - \frac{x_{max}}{N} |V_0| + \sum_{i\in V_0} \frac{x_i}{N} \\
   &=& tr(M^*) + \frac{ x_{max} }{N} \beta |V_1| + \sum_{i\in V_0} \frac{x_i}{N}  \\
   &\leq& tr(M^*) + \sum_{i}\frac{x_i M^*_{ii}}{N} + \frac{x_{max}}{N} \beta |V_1|
\end{eqnarray*}
where $\beta|V_1| = tr(M^*) - |V_0|$ is the fractional boundary part of $M^*$. This part can be 0 for some values of $\gamma$, or can be maximally $|V_1|$.
Note that $H^*$ is a fat shape in lattice, so the boundary is: $|V_1| = \Theta\left( \sqrt{tr(M^*)} \right) = O(1/\sqrt{\gamma})$.

Since $y_i = 1 + x_i/N$, we restore the solution by subtracting $tr(M^*)$ and then multiplying $N$, which gives:
\begin{equation*}
  c^*|_{H_0} \leq \sum_{i}x_i M^*_{ii} + x_{max} O(\sqrt{1/\gamma})
\end{equation*}

\end{proof}


\begin{lem}\label{lem:overlapping}
$G=(V,E)$ is a connected subgraph on an infinitely large 2D lattice. $G$ also satisfies:
\begin{enumerate}
  \item $|V| = \Omega(k)$;
  \item the conductance of $G$ is $\Theta(1/\sqrt{k})$:
\end{enumerate}
Then $G$ must contain a triangle of size $\Theta(k)$.
\end{lem}
\begin{proof}
We provide an intuitive sketch.
Consider all horizontal cuts on $G$. The most ``balanced'' horizontal cut $C_h$, where both parts are of size $\Theta(k)$, must have length $\Omega(\sqrt{k})$, otherwise (2) will be violated.
Consider all vertical cuts within the range of the balanced horizontal cut range.
Similar arguments follow that the most balanced vertical cut $C_v$ has size $\Omega(\sqrt{k})$.

Consider vertical cuts that start from $C_v$ and move aside stepwise along $C_h$.
Assume at some step the vertical cut passes through $a$ edges, the smaller part has $b$ nodes, and the conductance here is tight: $\frac{a}{b} = \Phi = \Omega(1/\sqrt{k})$. For the next vertical cut, assume the cut decreases by $\delta$ edges. The conductance at the new vertical cut is: $\frac{a-\delta}{b-a}\geq \Phi$.
Then we have $\frac{\delta}{a}\leq \Phi$, or $\delta = O(a \Phi) = O(1)$. This means that the shape can only contract by a constant number of nodes at each step, thus at least $\Theta(\sqrt{k})$ steps to shrink to 1 node. This triangle shape has size $\Theta(k)$.

\end{proof}

\end{document}